\documentclass{article} % For LaTeX2e
\usepackage{iclr2025_conference,times}

\usepackage{hyperref}
\usepackage{url}

\usepackage[T1]{fontenc}

\usepackage{booktabs}       % professional-quality tables
\usepackage{amsfonts}       % blackboard math symbols
\usepackage{nicefrac}       % compact symbols for 1/2, etc.
\usepackage{microtype}      % microtypography
\usepackage{xcolor}         % colors
\label{key}
\usepackage{thmtools}

\usepackage{comment}
\usepackage{amsmath}

\usepackage{epsfig}
\usepackage{graphicx}
\usepackage{wrapfig}
\usepackage{subfigure}
\usepackage{multirow}
\usepackage{hyperref}
\usepackage{graphicx}
\usepackage{caption}
\usepackage{array}
\usepackage{bbm}
\usepackage{color}
\usepackage{enumerate}
\usepackage{enumitem}
\setlist{leftmargin=10mm}
\usepackage{mathtools}
\usepackage{amsmath,amssymb,amsthm,bm}
\usepackage{amsfonts,graphicx}
\usepackage{mathrsfs}
\usepackage{algorithm}
\usepackage[noend]{algpseudocode}

\usepackage{tcolorbox}
\usepackage{multicol}

\newif\iffinal
\finaltrue

\iffinal
    \newcommand{\tianhao}[1]{}
    \newcommand{\ruoxi}[1]{}
    \newcommand{\rebuttal}[1]{#1}
\else
    \newcommand{\tianhao}[1]{{\bf \textcolor{purple}{[Tianhao: #1]}}}
    \newcommand{\ruoxi}[1]{{\bf \textcolor{violet}{[Ruoxi: #1]}}}
    \newcommand{\rebuttal}[1]{\textcolor{blue}{#1}}
\fi

\usepackage{cleveref}
\usepackage{wrapfig}

\newtheorem{theorem}{Theorem}

\newtheorem{definition}[theorem]{Definition}

\newtheorem{remark}{Remark}
\newtheorem{remark-star}{Remark}
\newtheorem{remark-star-1}{Remark}

\usepackage{booktabs}

\newcommand{\E}{\mathbb{E}}

\newcommand{\R}{\mathbb{R}}
\newcommand{\eps}{\varepsilon}

\newcommand{\norm}[1]{\left\lVert#1\right\rVert}

\newcommand{\g}{\nabla}
\newcommand{\iden}{\bm{I}}

\newcommand{\dataset}{\mathcal{D}}
\newcommand{\batch}{\mathcal{B}}
\newcommand{\batcht}{\batch_t}
\newcommand{\batchtstart}{\batch_{\start}}

\newcommand{\del}{\partial}

\newcommand{\hessian}{\mathbf{H}}

\newcommand{\start}{t_s}

\newcommand{\emb}{\texttt{DVEmb}}
\newcommand{\embtstart}{\emb^{(\start)}}
\newcommand{\embt}{\emb^{(t)}}

\newcommand{\embmiddle}[2]{\emb^{(#1 \to #2)}}
\newcommand{\KernelMatrix}[2]{\textbf{K}^{(#1 \to #2)}}

\newcommand{\zstar}{z^*}
\newcommand{\zval}{z^{(val)}}

\newcommand{\infl}{\textbf{IF}}

\newcommand{\outerprod}{\otimes}

\newcommand{\embM}{\mathbf{M}}

\newcommand{\projdimgrad}{\Tilde{p}}

\newcommand{\bs}{\textbf{s}}
\newcommand{\ba}{\textbf{a}}

\newcommand{\bW}{\textbf{W}}

\newcommand{\Pa}{\textbf{P}_{\ba}}
\newcommand{\Ps}{\textbf{P}_{\bs}}

\newcommand{\projdim}{r}

\newcommand{\commentcode}[1]{\texttt{// #1}}

\newcommand{\ta}{t_a}
\newcommand{\tb}{t_b}

\newcommand{\middlestep}{t_{\ell}}
\newcommand{\middleprevstep}{t_{\ell-1}}

\newcommand{\mO}{\mathcal{O}}

\newcommand{\unrollest}{\Delta \theta_{-\zstar}}

\newcommand{\seqtstart}{\sigma^{(\start)}}

\newcommand{\seqtrealization}{\sigma^{(t_r)}}

\newcommand{\expectedHessian}{\hessian_T^{(*)}}

\newcommand{\loo}{\texttt{TSLOO}}
\newcommand{\lootstart}{\loo^{(\start)}}

\newcommand{\eloo}{\texttt{ELOO}}

\newcommand{\A}{\mathcal{A}}

\title{Capturing the Temporal Dependence of Training Data Influence}

\author{
Jiachen T. Wang$^{\mathbf{\star}}$ \\
Princeton University \\
\AND
Dawn Song \\
UC Berkeley \\
\And 
James Zou \\
Stanford University \\
\And 
Prateek Mittal \\
Princeton University \\
\And 
Ruoxi Jia$^{\mathbf{\star}}$ \\
Virginia Tech \\
}

\iclrfinalcopy % Uncomment for camera-ready version, but NOT for submission.
\begin{document}

\maketitle

\newcommand\blfootnote[1]{%
  \begingroup
  \renewcommand\thefootnote{}\footnote{#1}%
  \addtocounter{footnote}{-1}%
  \endgroup
}

\begin{abstract}
Traditional data influence estimation methods, like influence function, assume that learning algorithms are permutation-invariant with respect to training data. However, modern training paradigms, especially for foundation models using stochastic algorithms and multi-stage curricula, are sensitive to data ordering, thus violating this assumption. This mismatch renders influence functions inadequate for answering a critical question in machine learning: How can we capture the dependence of data influence on the optimization trajectory during training? To address this gap, we formalize the concept of trajectory-specific leave-one-out (LOO) influence, which quantifies the impact of removing a data point from a specific iteration during training, accounting for the exact sequence of data encountered and the model's optimization trajectory. However, exactly evaluating the trajectory-specific LOO presents a significant computational challenge. To address this, we propose data value embedding, a novel technique enabling efficient approximation of trajectory-specific LOO. Specifically, we compute a training data embedding that encapsulates the cumulative interactions between data and the evolving model parameters. The LOO can then be efficiently approximated through a simple dot-product between the data value embedding and the gradient of the given test data. As data value embedding captures training data ordering, it offers valuable insights into model training dynamics. In particular, we uncover distinct phases of data influence, revealing that data points in the early and late stages of training exert a greater impact on the final model. These insights translate into actionable strategies for managing the computational overhead of data selection by strategically timing the selection process, potentially opening new avenues in data curation research.\blfootnote{$^\mathbf{\star}$Correspondence to \textbf{Jiachen T. Wang} and \textbf{Ruoxi Jia} (\texttt{tianhaowang@princeton.edu}, \texttt{ruoxijia@vt.edu}).}

\end{abstract}

% \vspace{-1mm}
\section{Introduction}

% \vspace{-1mm}
\textbf{Data influence estimation} aims to provide insights into the impact of specific data points on the model's predictive behaviors. Such understanding is crucial not only for model transparency and accountability \citep{koh2017understanding} but also plays a significant role in addressing AI copyright debates \citep{deng2023computational, wang2024economic} and facilitating fair compensation in data marketplaces \citep{tian2022private}. The majority of data influence estimation techniques focus on measuring the counterfactual impact of a training data point: \emph{how would the model's behavior change if we removed a specific training data point?}

% \vspace{-1mm}
\textbf{LOO \rebuttal{Influence}.} This counterfactual impact is often characterized by the \emph{Leave-One-Out} (LOO) \rebuttal{influence}, which has a long history and is frequently utilized in various fields such as robust statistics \citep{cook1980characterizations}, generalization analysis \citep{bousquet2002stability}, and differential privacy \citep{dwork2006calibrating}. Inheriting from this rich classical literature across various domains, the LOO influence in data influence studies is typically defined as 
$
\texttt{LOO}(\zstar; \zval) := \ell(\A(\dataset), \zval) - \ell(\A(\dataset \setminus \{\zstar\}), \zval)
$, i.e., the model's loss change on a validation data $\zval$ when the training data point $\zstar$ is removed from the training set $\dataset$. Here, $\A$ is the learning algorithm. 
For ease of analysis, traditional literature usually assumes that the learning algorithm $\A$ is permutation-invariant with respect to the training set $\dataset$, meaning that \emph{the order of data points does not affect the learning outcome} \citep{bousquet2002stability}. This assumption holds for models with strongly convex loss functions trained to converge.
% When using stochastic learning algorithms such as SGD, such an assumption still holds for models with strongly convex loss functions and are trained to converge. 
Within this framework, researchers have developed efficient methods to approximate LOO. 
Influence function \citep{koh2017understanding}, which uses first-order Taylor expansion to estimate the LOO, emerging as the most prominent approach. Numerous follow-up works have further improved its scalability for large models and datasets \citep{guo2021fastif,schioppa2022scaling,grosse2023studying, choe2024your}. 

% \textbf{The shift of training paradigm in foundation model era.} 
However, modern training algorithms, particularly those used for foundation models, increasingly deviate from the permutation-invariant assumption. This deviation arises from both the non-convex nature of neural networks and the multi-stage training curricula that do not run to convergence. In particular, due to the immense size of datasets, large language models (LLMs) often undergo just one training epoch, meaning each data point is encountered only once during training. Consequently, training data order significantly shapes the influence of data points on the final model \citep{epifano2023revisiting,nguyen2024bayesian}. Due to their underlying assumption of permutation-invariance, the order-dependence of data influence in modern training paradigms is not accurately reflected by influence functions. For example, they assign identical influence scores to duplicate training points, regardless of their position in the training sequence. 

% \vspace{-1mm}
Therefore, in this work, we argue that designing a data influence estimation technique relevant to the modern ML context requires rethinking how the counterfactual impact should be defined. Towards that end, we formalize the concept of \emph{trajectory-specific LOO}, which characterizes the loss change resulting from removing a data point from the specific iteration it is used during training. In contrast to the traditional LOO, trajectory-specific LOO explicitly accounts for the exact sequence of data encountered, considering the timing of a target training point being trained on. 
% and the model's state in the optimization trajectory when measuring counterfactual impact. 
An accurate evaluation of trajectory-dependent LOO would enable us to answer many important questions that are impossible to address with influence functions. For instance, how does a data point's impact vary depending on its entry timing in the training process? How do later points affect the influence of earlier points?

% \vspace{-1mm}
However, exactly evaluating the trajectory-specific LOO presents a significant computational challenge. To address this,
we introduce \textbf{data value embedding}, a novel data influence estimation framework designed for approximating trajectory-specific LOO. Our approach achieves several nice properties at the same time: \textbf{(1) accounting for training dynamics} and reflecting how the data order impacts model training; 
\textbf{(2) scale efficiently} to the setting of foundation models, and is faster than the current most efficient implementation of influence function;
% processing vast datasets and intricate architectures without imposing prohibitive computational costs.
\textbf{(3) enable real-time attribution} for any query without necessitating model retraining or prior access to validation data.

% \vspace{-1mm}
\textbf{Technical novelty.} 
Our proposed \emph{data value embedding} framework computes a compact representation for each data point that encapsulates the cumulative effect of subsequent training. The influence scores for any test instance can be approximated with a simple dot product operation between the test gradient and the data value embedding, enabling real-time computation of data influence scores. 
To improve the scalability of computing data influence embedding, we develop a suite of techniques for efficient computation and storage of data value embeddings. In particular, we introduce the \emph{influence checkpointing} technique, which enables the parallel computation of data value embeddings at multiple checkpoints. This not only enhances computational efficiency but also allows tracking of how a fixed data point's value changes during the training process.

% \vspace{-1mm}
\textbf{Empirical insights.} Through data value embedding, we obtain several novel empirical insights into the training dynamics of foundation models. 
% Our analysis reveals that the influence of individual data points on the final model is not uniform throughout training (Figure \ref{fig:influence-plot-pythia-Pile} (a)). Instead, w
We identified three distinct regimes of data influence (Figure \ref{fig:influence-plot-pythia-Pile} (a)): a very brief high-influence region at the start, a much longer low-influence basin, and a region in the later training stage with gradually increasing influence, resuming to a high level. We show that performing online data selection solely in the early and late high-influence regions (less than half of the training duration) can achieve performance improvements on par with selecting data throughout the entire process (Figure \ref{fig:influence-plot-pythia-Pile} (b)). Moreover, performing data selection \citep{fan2024doge} only in the first very brief high-influence region, lasting less than 4\% of the training duration, can achieve $\approx 50\%$ of the performance gain enabled by continuous selection. 
Since online data selection usually incurs significant computational costs, our findings suggest a viable way of managing this overhead by strategically timing the selection process. By focusing data selection efforts on these critical phases, we can substantially improve training efficiency without compromising model performance. These temporal insights can potentially embark on new avenues of research on budget-limited data curation.
% Our experiments show that performing online batch selection exclusively during these high-impact phases yields comparable model performance to online batch selection throughout the training process, while significantly reducing the computational overhead associated with data selection. These findings offer practical guidance for foundation model training, potentially leading to more resource-conscious development of AI systems.
\begin{figure}[h]
    \centering
    \setlength\intextsep{0pt}
    \setlength\abovecaptionskip{0pt}
    \setlength\belowcaptionskip{-10pt}
    \includegraphics[width=\textwidth]{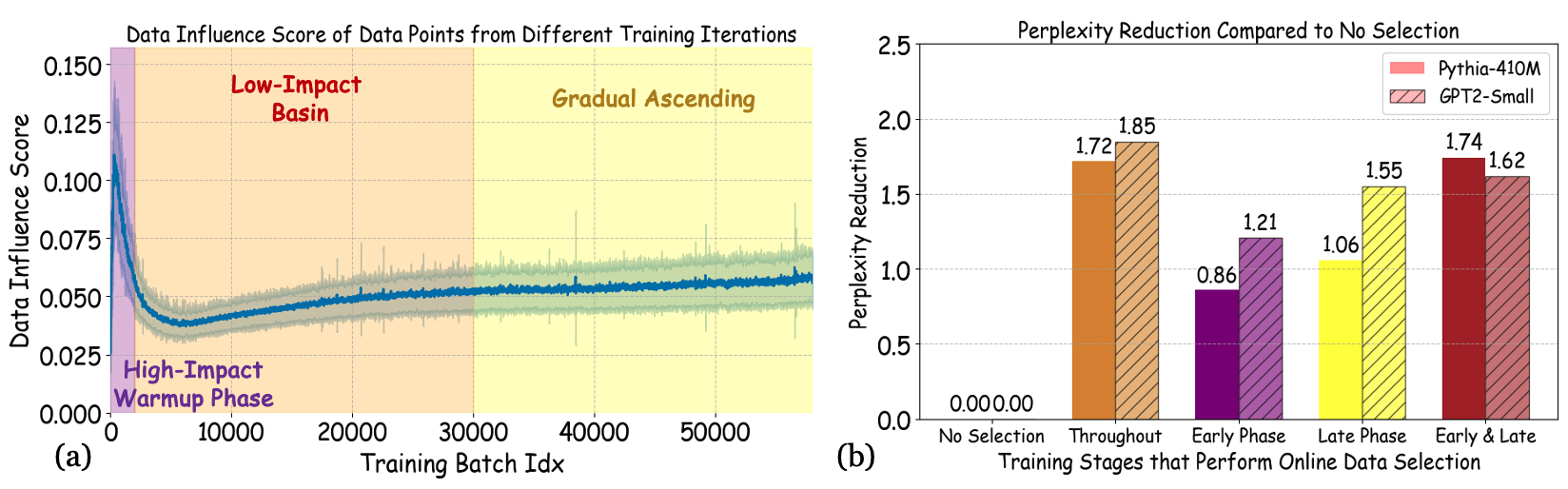}
    \caption{\textbf{(a)} Average data influence scores computed from data value embedding per training batch, measured against the final model's loss on Pile's validation set. Setting: Pythia-410M trained on 1\% of Pile. 
    \textbf{(b)} Comparison of online data selection strategies for training Pythia-410M on Pile. All strategies use gradient cosine similarity to Pile's validation set to select high-quality training batches \citep{fan2024doge}, and only differ in the training stages during which advanced batch selection is applied (random selection otherwise).
    }
    \label{fig:influence-plot-pythia-Pile}
\end{figure}

% \vspace{-2mm}
\section{Trajectory-Specific Leave-one-out Influence}

% \vspace{-2mm}
In this section, we formalize the definition of \emph{trajectory-specific LOO} \rebuttal{which was originally introduced in \citet{hara2019data} as 'SGD-influence'.} Consider a data point $\zstar$ that is included in the training process during the $\start$-th iteration. Let $\batch_t$ denote the training batch at iteration $t$. In standard SGD, the model parameters are updated as $\theta_{t+1} = \theta_t - \eta_t \sum_{z \in \batch_t} \nabla \ell(\theta_t, z)$ for $t = 0, \ldots, T-1$, where $\eta_t$ is the learning rate at iteration $t$. We are interested in the change in the validation loss $\ell(\theta_T, \zval)$ when the data point $\zstar \in \batchtstart$ is removed from iteration $\start$. 
In this counterfactual scenario, the parameter updates proceed as $\theta_{\start+1}' = \theta_{\start} - \eta_{\start} \sum_{z \in \batch_{\start} \setminus \{\zstar\}} \nabla \ell(\theta_{\start}, z)$ and $\theta_{t+1}' = \theta_t' - \eta_t \sum_{z \in \batch_t} \nabla \ell(\theta_t', z)$ for $t = \start+1, \ldots, T-1$. 
\begin{definition}[\textbf{Trajectory-Specific LOO} \rebuttal{\citep{hara2019data}}]
The \emph{trajectory-specific leave-one-out} for data point $\zstar$ at iteration $\start$ with respect to validation point $\zval$ is defined as
$$
\lootstart(\zstar; \zval) := \ell(\theta_T', \zval) - \ell(\theta_T, \zval)
$$
\end{definition}
\rebuttal{\textbf{Discussion.}} $\loo$ quantifies the change in validation loss resulting from removing $\zstar$ during the specific training run determined by the sequence of mini-batches and random initialization. $\loo$ explicitly depends on the timing of when the data is used and models the interaction effects between data points.
% , allowing it to capture nuanced aspects of the training process that traditional LOO approaches overlook. For example, $\loo$ can reveal the interaction effects between data points. 
For instance, it can show how the introduction of a certain type of example (e.g., a challenging edge case) might amplify or diminish the influence of previously seen, related examples. Moreover, identical data points contributing at different stages of training can receive different value scores. A data point introduced early in training might have a significantly different impact compared to the same point introduced later, as the model state evolves. 
However, traditional methods like influence functions do not capture these temporal dynamics. The influence function is defined as $\infl(\zstar; \zval) := \nabla_\theta \ell(\theta, \zval)^\top \hessian_{\theta}^{-1} \nabla_\theta \ell(\theta, \zstar)$ where $\hessian_{\theta}$ is the Hessian with respect to the full training loss. Because $\infl$ depends solely on the final state of the model, it invariably assigns the same influence value to identical $\zstar$s, regardless of their position in the training sequence.

% \vspace{-1mm}
\rebuttal{\textbf{Related works (extended version in Appendix \ref{appendix:related-work}).} 
Data attribution methods primarily fall into two categories: LOO-based methods and Shapley value-based methods. While Shapley value-based methods \citep{ghorbani2019data} offer elegant theoretical interpretation, they typically require expensive model retraining, which limits their practical applicability. As a result, LOO-based methods such as influence functions \citep{koh2017understanding} have gained more attention due to their computational efficiency. However, many studies have demonstrated that influence functions can be highly unreliable when applied to deep learning models \citep{basu2020influence,bae2022if,epifano2023revisiting}. In this work, we argue that $\loo$ provides a more appropriate attribution framework for deep learning, particularly in the context of foundation models. Various research communities have independently explored Taylor expansion-based technique (Section \ref{sec:method-unrolling}) for approximating $\loo$ for different purposes \citep{hara2019data,zou2021benign, evron2022catastrophic, wu2022power, wu2024many, ding2024understanding}. However, practical adoption has been hindered by computational demands. In this work, we propose a new method that overcomes the computational bottlenecks in approximating $\loo$ for large-scale models.}

% \vspace{-1mm}
\section{Data Value Embedding}
\label{sec:method}

% \vspace{-1mm}
% As discussed earlier, influence function overlooks the specifics of individual training runs. Alternative techniques like unrolled differentiation \citep{hara2019data} are explored in small-scale settings but are hindered by significant computational/memory demands. 
While trajectory-specific LOO offers clear benefits for understanding data influence in modern ML, its computation presents significant challenges. Exact computation is not feasible, as it would require removing a data point from a specific training iteration and re-initiating the entire training process.
To address this challenge, we introduce the concept of \emph{data value embedding}. % This embedding captures information about how a training data point interacts with other training points. The trajectory-specific LOO of a training data point can be approximated by simply projecting the test data's gradient (on which we want to study the influence) onto the training data's value embedding. 

% \vspace{-2mm}
\subsection{Preliminary: Unrolling the Effect of a Training Data Point in SGD}
\label{sec:method-unrolling}

% \vspace{-1mm}
Recall that we denote the final model as $\theta_T$ and the counterfactual model as $\theta_T'$, which is obtained by removing $\zstar$ from $\start$-th training iteration. We introduce an interpolation between $\theta_T$ and $\theta_T'$ by defining 
$
\theta_{\start+1}(\eps) := \theta_{\start} - \eta_{\start} \sum_{z \in \batch_{\start} \setminus \{\zstar\}} \nabla \ell(\theta_{\start}, z) - \eta_{\start} (1-\eps) \nabla \ell(\theta_{\start}, \zstar)
$ and 
$\theta_{k+1}(\eps) = \theta_k(\eps) - \eta_k \sum_{z \in \batch_k} \nabla \ell(\theta_k(\eps), z)$ for subsequent iterations. 
Note that $\theta_T(0) = \theta_T$ and $\theta_T(1) = \theta_T'$. 
Analogous to influence function-based approaches, we approximate the change in validation loss using a first-order Taylor expansion around $\eps = 0$: 
$
\ell(\theta_T', \zval) - \ell(\theta_T, \zval) \approx \nabla \ell(\theta_T, \zval)^\top \left.\frac{\partial \theta_T(\eps)}{\partial \eps}\right|_{\eps=0}
$. 
Interestingly, the derivative $\left.\frac{\partial \theta_T(\eps)}{\partial \eps}\right|_{\eps=0}$ satisfies a recursive relation detailed in Appendix \ref{appendix:unrolling-sgd}, and we can obtain a well-established approximation from the literature: 
\begin{align} \ell(\theta_T', \zval) - \ell(\theta_T, \zval) &\approx \eta_{\start} \nabla \ell(\theta_T, \zval)^\top \left[ \prod_{k=\start+1}^{T-1} (\iden - \eta_k \hessian_k) \right] \nabla \ell(\theta_{\start}, \zstar). 
\label{eq:sgdinf-expression}
\end{align}
where $\hessian_k = \sum_{z \in \batch_k} \nabla^2 \ell(\theta_k, z)$ is the Hessian and $\iden$ is the identity matrix. In data attribution literature, this approximation in (\ref{eq:sgdinf-expression}) first appears in \citet{hara2019data} and has also been utilized in \citet{chen2021hydra} and \citet{bae2024training}. 
It estimates the influence of removing $\zstar$ from the $\start$-th iteration on the validation loss $\ell(\theta_T, \zval)$ at the final iteration. 
The product term $\prod_{k=\start+1}^{T-1} (\iden - \eta_k \hessian_k)$ encapsulates the cumulative effect of the original data point's removal as it propagates through the entire training process. 
% This term accounts for both the learning rate and the curvature of the loss landscape at each subsequent step. 
Notably, similar product terms appear frequently in related domains, including continual learning and deep learning theory \citep{zou2021benign, evron2022catastrophic, wu2022power, wu2024many, ding2024understanding}.

% \vspace{-2mm}
\subsection{Data Value Embedding}
\label{sec:data-value-embedding}

% \vspace{-1mm}
Building on (\ref{eq:sgdinf-expression}), we extract the test-data-independent components and define \emph{"data value embedding"} for a training point $\zstar \in \batchtstart$ as 
\begin{align}
\embtstart(\zstar) := \eta_{\start} \left[ \prod_{k=\start+1}^{T-1} (\iden - \eta_k \hessian_k) \right] \nabla \ell(\theta_{\start}, \zstar)
\label{eq:embedding}
\end{align}
This embedding encapsulates the cumulative effect of a training point across the entire learning trajectory. By precomputing and storing these data value embeddings during or after the training phase, we enable highly efficient computation of data influence scores. Specifically, for any given test point $\zval$, the influence of a training point $\zstar$ can be quickly determined by simply computing the dot product $\nabla \ell(\theta_T, \zval)^\top \embtstart(\zstar)$. Vector dot products are among the most computationally efficient operations, especially when executed on modern GPU hardware, which is optimized for such parallelized vector operations. Precomputing the data value embeddings eliminates the need for costly retraining or the availability of test data in advance, making the computation of data influence nearly instantaneous. This is particularly advantageous in real-world scenarios such as data marketplaces, where rapid, on-demand data attribution is critical. 

\textbf{Approximation Error Bound.} 
In Appendix \ref{appendix:unroll-error-guarantee}, we derive a new theoretical analysis of the approximation error associated with the unrolled differentiation estimator for non-convex loss functions. We demonstrate that when the learning rate schedule satisfies $\eta_t \in \mO(1/\sqrt{t})$ with the maximum learning rate scaling as $\mO(1/\sqrt{T})$—a common choice in the literature \citep{vaswani2017attention}—the approximation error remains uniformly bounded and is \emph{independent} of the total number of training steps $T$. While the proof relies on certain assumptions to abstract the complexities of real-world implementation, the theoretical result still implies the method's applicability in practical model training.

\begin{figure}[t]
    \centering
    \includegraphics[width=0.75\textwidth]{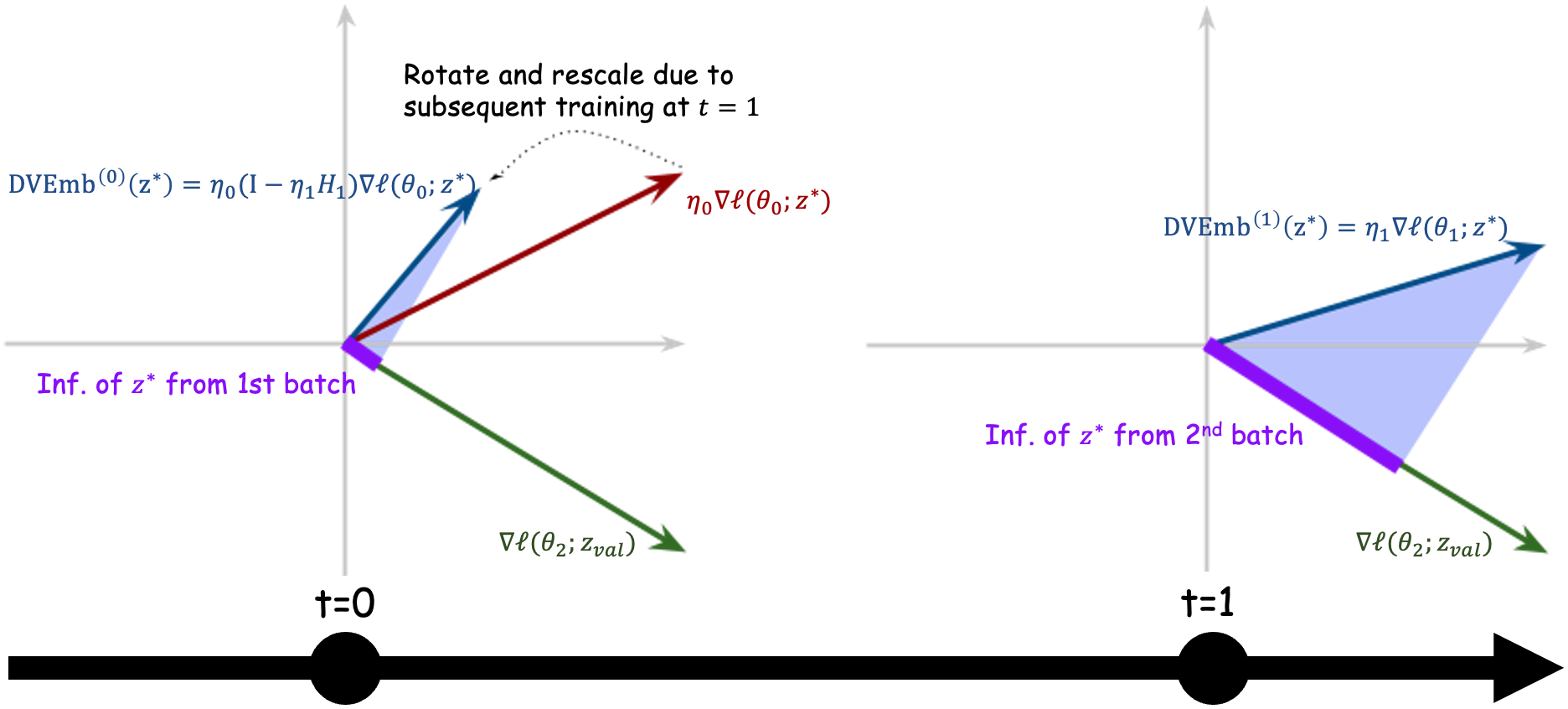}
    \caption{
    An illustrative example of data value embedding for a 2-step training. The influence of a training point $\zstar$ on a test point $\zval$ can be obtained by projecting its data value embedding on $\zval$'s gradient vector at the final checkpoint. %The data value embedding is computed across the learning trajectory, representing the accumulated effect of subsequent training.
    }
    \label{fig:geometric-interpretation}
\end{figure}

% \vspace{-2mm}
\section{Efficient Computation and Storage of Data Value Embedding}
\label{sec:efficient-computation}

% \vspace{-2mm}
While the data value embedding approach offers a promising solution for real-time data attribution that incorporates training-specific factors, its practical implementation faces significant computational and storage challenges. The computation of $\emb$ is non-trivial, requiring per-sample gradient calculations and per-step Hessian computations. Moreover, each $\emb_t(\zstar)$ has the same dimensionality as the model parameters, making it infeasible to store individual embeddings for each training data point on the disk. 
To address these challenges, we develop a series of techniques that significantly enhance both the computational and storage efficiency of data value embedding. 
% These improvements make our approach viable even for the large-scale architectures and vast datasets characteristic of modern foundation models.

% \vspace{-2mm}
\subsection{Recursive Approximation of Data Value Embedding via Generalized Gauss-Newton Matrix}

% \vspace{-2mm}
We show that data value embedding can be computed recursively, beginning from the final training iteration and working backward, when using the Generalized Gauss-Newton (GGN) approximation for the Hessian matrix. This naturally gives rise to a backward computation algorithm for $\embt$. 

% \vspace{-1mm}
A widely-adopted approximation for the Hessian matrix $\hessian_k$ is the Generalized Gauss-Newton (GGN) approximation $
\hessian_t 
\approx 
\sum_{z \in \batcht} \nabla \ell(\theta_t, z) \nabla \ell(\theta_t, z)^\top
$, particularly in the context of cross-entropy loss \citep{martens2020new}. 
The GGN approximation is extensively used in various machine learning algorithms because it captures the essential curvature information of the loss landscape while remaining computationally feasible. For further details, see Appendix \ref{appendix:GGN-approximation}. 
\rebuttal{Under this approximation to $\hessian_t$, the following shows that we can compute $\embtstart(\zstar)$ for any $\zstar \in \batchtstart$ if the data value embeddings of data points from later training iterations (i.e., $\embt(z)$ for $t \geq \start+1$) is available.}

\begin{theorem}
\label{thm:recursive-expression}
\rebuttal{
Given generalized Gauss-Newton approximation $
\hessian_t 
\approx 
\sum_{z \in \batcht} \nabla \ell(\theta_t, z) \nabla \ell(\theta_t, z)^\top
$, we have}
$$
\embtstart(\zstar) = \eta_{\start} \g \ell(\theta_{\start}, \zstar)
- \eta_{\start} 
\sum_{t=\start+1}^{T-1} 
\left(
\sum_{z \in \batch_{t}} 
\left( \g \ell(\theta_{t}, z)^\top 
\g \ell(\theta_{\start}, \zstar) \right)
\emb^{(t)} (z)
\right)
$$
\end{theorem}
The proof is deferred to Appendix \ref{appendix:recursive}. 

% \vspace{-1mm}
\rebuttal{\textbf{Interpretation.} Theorem \ref{thm:recursive-expression} provides crucial insights into the interactions between training data points throughout the model training process. 
When two points $\zstar$ and $z$ are similar, their gradient similarity term $\nabla \ell(\theta_t, z)^\top \nabla \ell(\theta_{\start}, \zstar)$ increases, indicating stronger interaction between these points. To illustrate this phenomenon, consider training a language model where an early data point $\zstar$ contains content about "quantum computing". The influence of $\zstar$ on the final model varies depending on the subsequent training data: if multiple similar "quantum computing" data points appear in later iterations, $\zstar$'s influence on the final model diminishes, as these later examples could teach similar concepts to the model. Conversely, if $\zstar$ remains one of the few "quantum computing" examples throughout training, it maintains a stronger influence on the final model.}

% \vspace{-1mm}
\textbf{Overview of the remaining sections.} 
Theorem \ref{thm:recursive-expression} suggests the possibility of a backpropagation algorithm for computing data value embeddings, contingent on the availability of per-sample gradient vectors for all training data. To make this approach practical for large-scale applications, we address two key challenges in the following sections: (1) Efficient computation and storage of per-sample gradient vectors for all training data (Section \ref{sec:store-gradient}). (2) Efficient computation (Sections \ref{sec:alg-backpropagation}) and parallelization (Section \ref{sec:method-influence-cpts}) of data value embeddings using Theorem \ref{thm:recursive-expression}. Additionally, we discuss practical extensions and considerations for real-world scenarios (Appendix \ref{appendix:practical-consideration}).

% \vspace{-1mm}
\subsection{Step 1: Store Per-Sample Training Gradient Information at Each Iteration}
\label{sec:store-gradient}

% \vspace{-1mm}
During model training, we additionally store the \emph{per-sample} gradient for each data point in the training batch. However, this approach presents significant computational and storage challenges: \textbf{(1) Storage:} Let $p$ denote the number of model parameters. Each gradient vector has dimension $p$, requiring $\mO(TBp)$ disk space, where $B=|\batcht|$ is the batch size. This effectively corresponds to storing millions of model-size vectors. \textbf{(2) Efficiency:} Computing per-sample gradients necessitates separate backpropagation for each $z \in \batcht$, increasing computational cost by a factor of $B$.

% \vspace{-1mm}
\textbf{Avoiding per-sample gradient computation \& full gradient storage (detailed in Appendix \ref{appendix:ghost}).} To mitigate both issues, we leverage a gradient decomposition and take advantage of the computations already performed during backpropagation \citep{wang2024data, choe2024your}. By expressing gradients as the outer product of activations and output derivatives, only a single backpropagation on the aggregated loss is required to compute per-sample gradients, preserving the usual training speed. Additionally, instead of storing the full gradient vectors, we store the decomposed components, potentially reducing the storage requirement to $\mO(TB\sqrt{p})$ for non-sequential data.

% \vspace{-1mm}
\textbf{Random projections for large models.} For large-scale foundation models with billions of parameters, we apply random projections to further compress the stored gradient information. Using projection matrices, we project the activations and output derivatives to a lower-dimensional space. This approach significantly reduces storage needs to $\mO(TB\projdimgrad)$, where $\projdimgrad$ is the projected dimension, while still capturing essential gradient geometric information.

% \vspace{-1mm}
We acknowledge that deriving a theoretical multiplicative guarantee here is challenging, given that the data value embedding itself is a linear combination that could be zero. However, our ablation study in Appendix \ref{appendix:eval-projection-error} demonstrates that our approach is relatively more robust compared to influence functions across different projection dimensions. These results provide strong evidence of the robustness of our method in practice, and we leave the theoretical guarantee as future work.

% \vspace{-2mm}
\subsection{Step 2: Backpropagating Data Value Embedding}
\label{sec:alg-backpropagation}

% \vspace{-1mm}
Having established the method for storing projected gradient vectors, we now proceed to describe the backward computation algorithm for data value embeddings. For ease of presentation, we continue to use full gradient vector notation. However, in practical implementations, we use the projected gradient vectors for efficient storage. That is, $\g_\theta \ell \in \R^{\projdimgrad}$ in the subsequent contents.

% \vspace{-1mm}
According to Theorem \ref{thm:recursive-expression}, an equivalent expression for $\embtstart(\zstar)$ is given by 
$$
\embtstart(\zstar)
= \eta_{\start} \g \ell(\theta_{\start}, \zstar)
- \eta_{\start} 
\g \ell(\theta_{\start}, \zstar) \embM^{(\start)}
$$ where $\embM^{(\start)} := \sum_{t=\start+1}^{T-1} \left( \sum_{z \in \batch_{t}} \left( \emb^{(t)}(z) \nabla \ell(\theta_{t}, z)^\top \right) \right)$. 
At a high level, our algorithm computes $\embtstart(\zstar)$ for each $\start$ from $T-1$ down to $0$, while maintaining a running matrix $\embM^{(\start)} \in \R^{\projdimgrad \times \projdimgrad}$ throughout the backpropagation process for algorithm efficiency.

% \vspace{-1mm}
\textbf{Backward algorithm from the final iteration.} 
We initialize $\embM^{(T-1)} = \mathbf{0}$ as the data value embedding coincides with the training gradient for the last iteration. For $\start = T-1, \ldots, 0$, we recursively compute: \textbf{(1)} The data value embedding for each $\zstar \in \batchtstart$: 
$
\embtstart(\zstar) = 
\eta_{\start} \nabla \ell(\theta_{\start}, \zstar)
- \eta_{\start} \embM^{(\start)} \nabla \ell(\theta_{\start}, \zstar)
$, and \textbf{(2)} Update the weighting matrix after computing all embeddings for the current iteration: 
$
\embM^{(\start-1)} =
\embM^{(\start)} + 
\sum_{\zstar \in \batchtstart} 
\emb^{(\start)}(\zstar) \g \ell(\theta_{\start}, \zstar)^\top
$. 
A detailed algorithm pseudocode can be found in Algorithm \ref{alg:full-backpropagate}.

% \vspace{-1mm}
\textbf{Computing data value embedding on a per-layer basis.} Moreover, by adopting an assumption similar to that in EK-FAC regarding the independence of gradients across different layers, we can compute data value embeddings on a per-layer basis. This approach significantly reduces the computational and memory costs. The assumption of layer-wise independence is common in the literature on influence functions \citep{grosse2023studying}, as it enables tractable analysis and efficient algorithms for deep neural networks. While this approximation neglects cross-layer gradient correlations, it is often justified because intra-layer interactions tend to dominate in practice. Treating layers independently thus strikes a favorable balance between computational feasibility and approximation accuracy.

% \vspace{-1mm}
\textbf{Complexity analysis.} 
\textbf{(1) Computational \& Memory:} The primary computational cost of our algorithm stems from matrix multiplications and additions in updating data value embeddings and the weighting matrix, resulting in $\mO(BT\projdimgrad^2)$ floating-point operations (flops). However, if we compute the data value embedding per layer, flops improve to $\mO(BT\projdimgrad^2/L)$ where $L$ is the number of layers. 
The update of the running matrix $\embM^{(\start)}$ requires $\mO(B
\projdimgrad^2/L^2)$ memory. In comparison, regular model training requires $\mO(BTp)$ flops and $\mO(p)$ memory, where $p$ is the number of model parameters. 
%Importantly, the required projection dimension $\projdimgrad$ is of the order $\mO(\log(BT))$. In most scenarios, the dataset size ($\mO(BT)$) is typically much smaller than the number of model parameters ($p$), leading to $\projdimgrad \ll \sqrt{p}$. 
Consequently, Algorithm \ref{alg:full-backpropagate} incurs significantly lower costs compared to regular training. We further note that the influence function method requires computing the per-sample gradient for each training data point on the final model, which is effectively equivalent to one epoch of training. As a result, both the memory requirements and flops for the influence function method are at least equivalent to those of model training, which are much larger than our algorithm's requirements. 
\textbf{(2) Storage:} Each $\embt(\zstar)$ has dimension $\mO(\projdimgrad)$, resulting in a total storage requirement of $\mO(BT\projdimgrad)$ for data value embeddings across all training points. While this can be substantial, disk storage is relatively inexpensive in modern computing environments. Moreover, the reduced dimensionality achieved through projection significantly mitigates the storage burden compared to storing full-dimensional embeddings. A summary of the complexity comparison with the most efficient implementation of the influence function \citep{choe2024your} is provided in Table \ref{tb:complexity} \rebuttal{in Appendix \ref{appendix:method-complexity}.}

% \vspace{-2mm}
\subsection{Parallelized Extension for Influence Embedding Computation (Overview)}
\label{sec:method-influence-cpts}

% \vspace{-2mm}
The backpropagation algorithm introduced in Section \ref{sec:alg-backpropagation} operates with a runtime complexity of $\mO(T)$, as it sequentially computes $\embtstart$ for $\start = T-1, \ldots, 0$. While being significantly more efficient than the influence function, which requires re-computing all training gradients on the final model (see Section \ref{sec:eval-efficiency} and Table \ref{tb:complexity}), it can still be costly for long training periods. 
Here, we introduce \emph{influence checkpointing}, a parallelized extension for Algorithm \ref{alg:full-backpropagate}. 

% \vspace{-1mm}
\textbf{Influence Checkpointing.} We reduce computational costs by allowing concurrent computation of data value embeddings at multiple checkpoints during training. By selecting $K$ evenly spaced training steps, we can efficiently compute data value embeddings for each \emph{intermediate checkpoint} in parallel. By carefully computing and storing necessary results, we can efficiently reconstruct the data value embedding for the final model. This reduces the overall computational cost by $K$ times. The detailed algorithm description, pseudocode, and complexity analysis are deferred to Appendix \ref{appendix:influence-checkpointing}.

% \vspace{-1mm}
\textbf{Data Value Dynamics During Training.} 
In addition to its computational benefits, the influence checkpointing algorithm enables a powerful capability: tracking the evolution of data influences throughout the entire model training process. If the intermediate checkpoints $\theta_{t_1}, \ldots, \theta_{t_{K-1}}$ was saved—a common practice in foundation model pretraining—we can analyze how the influence of a fixed data point changes on different intermediate checkpoints. As a result, we gain a more fine-grained and dynamic view of how the influence of a fixed data point propagates to the subsequent training steps, providing deeper insights into the model's learning behavior over time. This capability opens up new avenues for understanding and optimizing machine learning model training.

% \vspace{-2mm}
\section{Experiments}
\label{sec:eval}

% \vspace{-2mm}
In this section, we evaluate the effectiveness of our proposed data value embedding method. First, we assess its fidelity in accurately reflecting data importance using small-scale experimental setups (Section \ref{sec:eval-fidelity}), as well as its computational efficiency (Section \ref{sec:eval-efficiency}). We then apply data value embedding to analyze the training dynamics during foundation model pretraining (Section \ref{sec:eval-analysis} and \ref{sec:eval-qualitative}). The baselines, implementation details, and additional results are deferred to Appendix \ref{appendix:eval}.

% \vspace{-2mm}
\subsection{Fidelity Evaluation}
\label{sec:eval-fidelity}

% \vspace{-2mm}
To validate the effectiveness of our proposed data value embedding algorithm, we assess its accuracy in approximating $\loo$ scores. \rebuttal{Additionally, in Appendix \ref{appendix:mislabel-detection}, we compare to a variety of data attribution baselines on the standard benchmarks of mislabel data detection and data selection.}

% \vspace{-1mm}
Computing ground-truth LOO requires retraining the model multiple times, each time excluding a single data point while keeping all other training specifics, such as batch order, unchanged. Given the computational intensity, we conduct our experiments on the MNIST \citep{lecun1989handwritten} using a small MLP trained with standard SGD. 
We consider two settings: \textbf{(1) Single epoch removal}, where a data point is excluded from training during a single epoch but still in other training epochs. Here, we remove the data point from the last epoch. \textbf{(2) All-epoch removal}, where a data point is excluded in all epochs. In this case, the approximation provided by data value embedding is obtained by summing the data value embeddings of the data point from all epochs, as discussed in Appendix \ref{appendix:practical-consideration}.

% \vspace{-1mm}
Figure \ref{fig:groundtruth-LOO-comparison} shows that data value embedding has a high Spearman correlation with the ground-truth LOO. This superior performance is consistent across both settings. 
%, highlighting our algorithm's efficacy in approximating performance changes in data group removal scenarios. 
We note that the influence function scores remain constant for both settings, as influence functions do not account for specific training runs and cannot differentiate between single- and multi-epoch removals. Moreover, influence function exhibits a very weak correlation with LOO, a phenomenon that has been reported in many literature \citep{sogaard2021revisiting,basu2020influence,bae2022if,epifano2023revisiting}. 
% \tianhao{add group removal if have time.}

\begin{figure}[t]
    \centering
    \setlength\intextsep{0pt}
    \setlength\abovecaptionskip{0pt}
    \includegraphics[width=\textwidth]{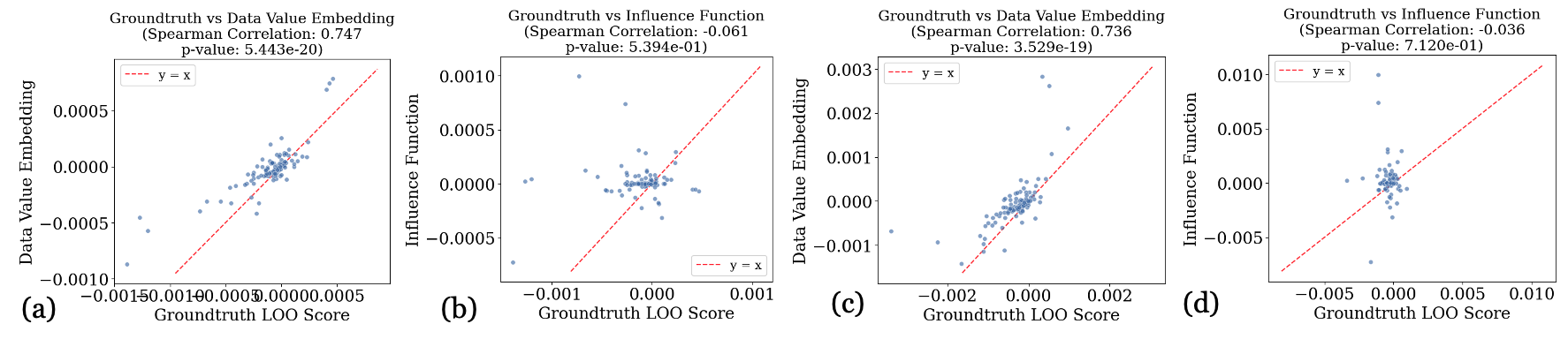}
    \caption{
    The correlation between ground-truth LOO when the MLP is trained for 3 epochs and the estimation obtained by (a) the data value embedding method and (b) the influence function for \emph{single epoch removal}. (c) and (d) present the corresponding correlations for \emph{all-epoch removal}. 
    Additional results for models being trained for a longer time can be found in Appendix \ref{appendix:eval-fidality}. 
    }
    \label{fig:groundtruth-LOO-comparison}
\end{figure}

% \vspace{-2mm}
\subsection{Computational Efficiency}
\label{sec:eval-efficiency}

% \vspace{-2mm}
In this section, we compare the storage, memory, and computational efficiency of data value embedding with LoGRA \citep{choe2024your}, the most efficient implementation of the influence function so far. LoGRA first computes per-sample training gradients on the final model for \emph{all} training data points $\zstar \in \dataset$, where $\dataset$ represents the dataset. Like our algorithm, LoGRA also uses random projection and stores the \emph{projected} Hessian-adjusted gradient $\hessian_T^{-1} \g \ell(\theta_T, \zstar)$ to the disk, and the influence function can be computed via dot-product with test data gradient.

% \vspace{-1mm}
Table \ref{tb:efficiency} shows the result of computing data influence for Pythia-410M trained on 1\% of the Pile dataset. Both algorithms first compute and store Hessian-adjusted gradients/data value embedding, and then compute the data influence with respect to any given test point. 
As we can see, LoGRA and data value embedding have similar disk storage requirements, as both approaches save vectors of dimension $\projdimgrad$ for each data point. For peak GPU memory in the storage step, LoGRA requires recomputing gradients for all training data on the final model $\theta_T$, which is effectively equivalent to one epoch of model training. In contrast, the data value embedding computation algorithm operates only on projected vectors, which takes much less GPU memory (0.84 vs 63.6GB). Consequently, the computational efficiency for computing data value embeddings is also much higher (over $15 \times$ faster). 
% We can also view LoGRA as performing one model retraining, as it needs to backpropagate all data points on the final model checkpoint $\theta_T$. 
When computing data influence, since both approaches simply take the dot product between test data's (projected) gradient and $\hessian_T^{-1} \g \ell(\theta_T, \zstar)$ or $\embt(\zstar)$ or data value embedding, the GPU memory usage and efficiency are the same.

\begin{table}[h]
\centering
\setlength\intextsep{0pt}
\setlength\abovecaptionskip{2pt}
\resizebox{\columnwidth}{!}{\begin{tabular}{@{}cccccc@{}}
\toprule
\textbf{}                     & \multicolumn{3}{c}{\textbf{Storing $\hessian_T^{-1} \g \ell(\theta_T, \zstar)$ / data value embedding}}         & \multicolumn{2}{c}{\textbf{Compute Influence (dot-product)}} \\ \midrule
\textbf{}                     & \textbf{Storage} & \textbf{Peak GPU Mem.} & \textbf{Throughput} & \textbf{Peak GPU Mem.}      & \textbf{Throughput}     \\
\textbf{LoGRA}                & 170GB            & 63.6GB          & 41.6                & 16.31GB              & 640                     \\
\textbf{Data Value Embedding} & 171GB            & 64.6GB / 0.84GB* & 667.52              & 16.31GB              & 640                     \\ \bottomrule
\end{tabular}}
\caption{
Memory and compute efficiency analysis for LoGRA \citep{choe2024your} and data value embedding. Throughput is measured as the number of data points per second for storing and influence computation. The experiment is conducted on one A100 GPU with 80GB VRAM. 
The projection dimension is set to 1024.
*Since data value embedding technique contains two different steps in storing relevant information for data attribution (storing gradient during training \& compute and store data value embedding after training), we include the peak GPU memory usage for both steps.
}
\label{tb:efficiency}
\end{table}

% \vspace{-2mm}
\subsection{Analyzing Training Dynamics of Foundation Models}
\label{sec:eval-analysis}

% \vspace{-2mm}
In this section, we showcase data value embedding as a powerful tool for analyzing the training dynamics of foundation model pretraining with Pythia-410M trained on 1\% of Pile dataset as an example. 
% This study demonstrates the unique capability of data value embedding to reveal the evolution of data influence throughout the model training process, distinguishing our approach from previous works by providing insights into the dynamic nature of data importance during training. 
Results for additional datasets/models and the analysis for fine-tuning are in Appendix \ref{appendix:eval-dynamics}.

% \vspace{-1mm}
\textbf{Value of training data from different stages in LLM pretraining.} We first visualize the distribution of data influence scores on the final model across different training batches. 
For a fair comparison, we normalize the influence scores for each batch by their learning rate. 
Figure \ref{fig:influence-plot-pythia-Pile} (a) illustrates the results for training Pythia-410M on the Pile dataset. 
As we can see, the data influence on the final model can be categorized into three distinct regimes: 
\textbf{(1) High-impact Warmup Phase:} This phase occurs during the very early training stage and is characterized by exceptionally high data influence scores. It corresponds to a brief window at the onset of training where the loss reduces rapidly. 
\textbf{(2) Low-impact Basin:} This regime spans the early-to-middle training stage, where data influence scores are significantly lower. This period coincides with a slowdown in the rate of loss decrease, transitioning into a phase of relative stability.
\textbf{(3) Gradual Ascent:} In this phase, we observe that the later a data point participates in the training, the higher its influence score becomes.

% \vspace{-1mm}
\textbf{Explanation:} 
\textbf{(1) Parameter initialization and warmup training are important for final model performance.} 
During the very early stages of training, the gradient norms are large, which leads to significant parameter updates. Furthermore, the subsequent gradients' magnitude \emph{decrease rapidly}, causing data points from the High-impact Warmup Phase to maintain substantial influence throughout the training process, even as their immediate impact diminishes over time. 
\begin{wrapfigure}{r}{0.47\textwidth}
    \centering
    \setlength\intextsep{0pt}
    \setlength\abovecaptionskip{0pt}
    \setlength\belowcaptionskip{-10pt}
    % \vspace{-5mm}
    \includegraphics[width=0.48\textwidth]{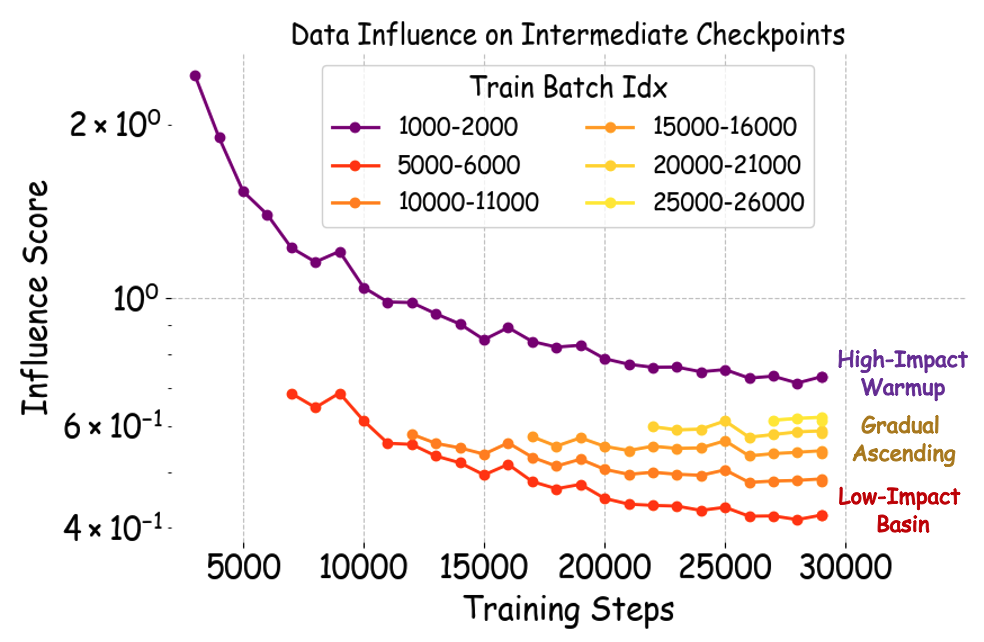}
    \caption{\rebuttal{Evolution of influence scores across training checkpoints. The x-axis shows training iterations, and the y-axis shows the average influence of training examples on each checkpoint. Examples are grouped according to the iterations they are being trained on.}}
    \label{fig:influence-saturation}
\end{wrapfigure}
Figure \ref{fig:influence-saturation} visualizes this phenomenon. The \rebuttal{purple curve} shows that training data points from the High-impact Warmup Phase, while experiencing large drops in influence as training progresses, still maintain higher influence than later data points. This observation aligns with the well-known effect that model initialization and/or warm-up training plays a crucial role in training performance \rebuttal{\citep{he2015delving, hanin2018start}}, effectively initializing model parameters and gradually preparing the model for more complex learning tasks.
\textbf{(2) Influence saturation from future data.} As training progresses into a smoother loss regime, the gradient norms become relatively stable and decrease slowly. This makes the influence decay from subsequent training much more significant for these data points compared to those from the High-Impact Warmup Phase. Since earlier data points experience more future training iterations, their influence decreases more over time. The \rebuttal{red curve} in Figure \ref{fig:influence-saturation} demonstrates this trend, showing influence scores for these points gradually decreasing during training and eventually falling below those of later training data points. One might initially think this phenomenon is connected to catastrophic forgetting, where the model appears to "forget" the influence of data from earlier training phases as it progresses. However, we note that a data point's influence score decreases the most when future data points are similar to it, which is different from catastrophic forgetting. Intuitively, if future points are identical, the presence of the earlier data point in training becomes less relevant to the model's behavior. A more detailed explanation is deferred to Appendix \ref{appendix:eval-dynamics}.

% \vspace{-1mm}
\textbf{Implications for data selection strategies.} 
These observations suggest that for pretraining, data selection is most critical during the very early and later stages of training. To validate this insight, we train Pythia-410M on Pile with different online data selection strategies, as shown in Figure \ref{fig:influence-plot-pythia-Pile} (b). \rebuttal{Specifically, we use an online data selection strategy (adapted from \citet{fan2024doge}) that forms each training batch by selecting data points whose gradients align well with those from a validation batch sampled from Pile (see Appendix \ref{appendix:eval-dataselection-detail} for details). This selection process requires computing gradient similarities, introducing significant overhead at each iteration where it is applied. Therefore, identifying the most critical training phases for applying this selection process becomes crucial for computational efficiency.} 
Remarkably, Figure \ref{fig:influence-plot-pythia-Pile} (b) demonstrates that the performance of a strategy where we only perform data selection in the first 2000 iterations and after 20000 iterations closely matches the performance when data selection is performed in all iterations. Moreover, it reduces computational costs by more than 5 times. 
This corroborates our practical insights for designing efficient data selection strategies in LLM pretraining: by focusing data selection efforts on the critical early and late stages of training, we can potentially achieve optimal model performance while significantly reducing computational overhead. 
% We also note that a common curriculum learning strategy is to introduce easy samples early and harder data later often prove effective. The easy samples serve as a warm-up, aligning with the high-influence early training phase we identified.

% \vspace{-2mm}
\subsection{Qualitative Evaluation}
\label{sec:eval-qualitative}

% \vspace{-2mm}
We conduct a qualitative analysis to examine the similarities between a test data point $\zval$ and the most valuable data points identified by data value embedding. In this experiment, we set $\zval$ to be identical to one of the training data points, making the most similar data point its own repetition. In data valuation literature, the influence score of a training point on its repetition is usually referred to as "self-influence" \citep{koh2017understanding} and is being used to measure memorization \citep{feldman2020neural}. Intuitively, the self-influence should be the highest among all training points.

\definecolor{lightgray}{gray}{0.9}
\definecolor{lightorange}{rgb}{1, 0.85, 0.6}

% \vspace{-1mm}
Figure \ref{fig:corpus-similarity} shows representative results from training GPT-2 on Wikitext-103 over three epochs, where the test data is about \colorbox{green}{military video game}. 
As observed, for model checkpoints after the 2nd and 3rd epochs, the \colorbox{green}{test data point's repetition} achieves the highest influence score, as expected. 
However, for the model checkpoint after the 1st epoch, the most valuable data points are not the repetition but rather a similar data about \colorbox{lightorange}{war history}. This discrepancy occurs because, during the first epoch of training, the repetition of the test data point resides in the \emph{low-value basin} identified in Section \ref{sec:eval-analysis}, resulting in a lower self-influence score as subsequent training progresses. 
Additionally, we observe that the influence function may incorrectly identify irrelevant data points as highly influential (e.g., the \colorbox{lightgray}{Popular Music} completely irrelevant to \colorbox{green}{military video game} but being identified as the second most valuable data), possibly due to its bias towards data points with high gradient norms, as also noted in \citet{barshan2020relatif}. This limitation underscores the advantages of data value embedding in providing more accurate and context-aware data influence assessment.
\begin{figure}[h]
    \centering
    \setlength\intextsep{0pt}
    \setlength\abovecaptionskip{0pt}
    \includegraphics[width=0.99\textwidth]{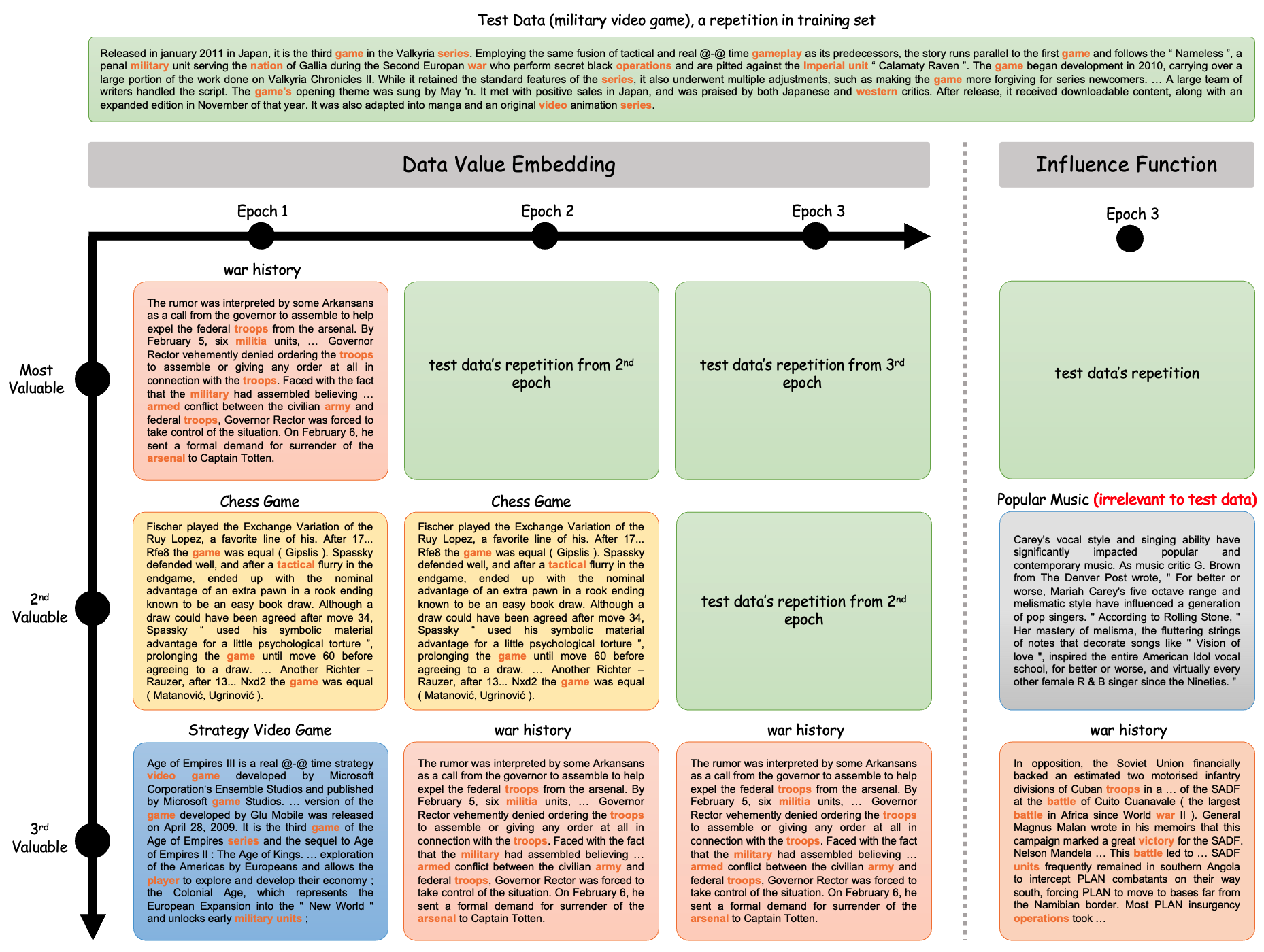}
    \caption{
    Visualization of (left) the evolution of the top-3 most valuable training data points identified by data value embedding throughout 3 training epochs and (right) the top-3 most valuable training data points identified by influence function. We use GPT-2 trained on Wikitext-103, with the test point being a repetition of a training data point related to a military video game. The common words between the test and training data are highlighted in \colorbox{orange}{orange}.
    }
    \label{fig:corpus-similarity}
\end{figure}

% \vspace{-2mm}
\section{Conclusion and Limitations}

% \vspace{-2mm}
In this paper, we introduced Data Value Embedding, a novel approach to data attribution tailored for foundation models. Our method addresses critical limitations of existing techniques by capturing the temporal dynamics of training and enabling real-time attribution without the need for model retraining. The experiments demonstrate the efficacy of data value embedding in providing accurate and efficient data influence scores and unveiling unique insights into the training dynamics of foundation models. 

% \vspace{-1mm}
\textbf{Limitations: SGD as a proxy for Adam.} The data value embedding in (\ref{eq:embedding}) is specifically tailored for SGD. It is not directly extendable to other popular optimizers like Adam due to their normalization terms. Nonetheless, using SGD as a proxy for Adam allows for efficient data influence estimation, which is the approach that is usually adopted in practice and has proved to be effective in our experiment, providing a practical and effective solution for the current scope of our work. While using as a proxy for Adam has proved to be effective in our experiment, extending data value embedding to Adam and other optimizers remains an exciting direction for future research.

\textbf{Potential future work: training curriculum design.} Our findings on the varying influence of data points across training stages suggest the potential for designing optimal training curricula. Future work could explore leveraging data value embedding to design curricula that maximize learning efficiency. This could involve dynamically adjusting the presentation order and frequency of data points based on their predicted influence at different training stages.

\clearpage

\section*{Acknowledgment}

This work is supported in part by the National Science Foundation under grants IIS-2312794, IIS-2313130, OAC-2239622, CNS-2131938, CNS-2424127, Amazon-Virginia Tech Initiative in Efficient and Robust Machine Learning, the Commonwealth Cyber Initiative, Cisco, OpenAI and Google.

We thank Tong Wu, Meng Ding, and Weida Li for their helpful feedback on the preliminary version of this work.

% \section*{Reproducibility Statement}
% The experiment settings are detailed in Appendix \ref{appendix:eval}. All theorem proofs are provided in Appendix \ref{appendix:algorithm-details}. 
% We will release the source code upon publish. 

\bibliography{ref}

\begin{thebibliography}{76}
\providecommand{\natexlab}[1]{#1}
\providecommand{\url}[1]{\texttt{#1}}
\expandafter\ifx\csname urlstyle\endcsname\relax
  \providecommand{\doi}[1]{doi: #1}\else
  \providecommand{\doi}{doi: \begingroup \urlstyle{rm}\Url}\fi

\bibitem[Amiri et~al.(2022)Amiri, Berdoz, and Raskar]{amiri2022fundamentals}
Mohammad~Mohammadi Amiri, Frederic Berdoz, and Ramesh Raskar.
\newblock Fundamentals of task-agnostic data valuation.
\newblock \emph{arXiv preprint arXiv:2208.12354}, 2022.

\bibitem[Bae et~al.(2022)Bae, Ng, Lo, Ghassemi, and Grosse]{bae2022if}
Juhan Bae, Nathan Ng, Alston Lo, Marzyeh Ghassemi, and Roger~B Grosse.
\newblock If influence functions are the answer, then what is the question?
\newblock \emph{Advances in Neural Information Processing Systems}, 35:\penalty0 17953--17967, 2022.

\bibitem[Bae et~al.(2024)Bae, Lin, Lorraine, and Grosse]{bae2024training}
Juhan Bae, Wu~Lin, Jonathan Lorraine, and Roger Grosse.
\newblock Training data attribution via approximate unrolled differentation.
\newblock \emph{arXiv preprint arXiv:2405.12186}, 2024.

\bibitem[Barshan et~al.(2020)Barshan, Brunet, and Dziugaite]{barshan2020relatif}
Elnaz Barshan, Marc-Etienne Brunet, and Gintare~Karolina Dziugaite.
\newblock Relatif: Identifying explanatory training samples via relative influence.
\newblock In \emph{International Conference on Artificial Intelligence and Statistics}, pp.\  1899--1909. PMLR, 2020.

\bibitem[Bartlett(1953)]{bartlett1953approximate}
MS~Bartlett.
\newblock Approximate confidence intervals.
\newblock \emph{Biometrika}, 40\penalty0 (1/2):\penalty0 12--19, 1953.

\bibitem[Basu et~al.(2020)Basu, Pope, and Feizi]{basu2020influence}
Samyadeep Basu, Philip Pope, and Soheil Feizi.
\newblock Influence functions in deep learning are fragile.
\newblock \emph{arXiv preprint arXiv:2006.14651}, 2020.

\bibitem[Bousquet \& Elisseeff(2002)Bousquet and Elisseeff]{bousquet2002stability}
Olivier Bousquet and Andr{\'e} Elisseeff.
\newblock Stability and generalization.
\newblock \emph{The Journal of Machine Learning Research}, 2:\penalty0 499--526, 2002.

\bibitem[Burgess \& Chapman(2021)Burgess and Chapman]{burgess2021approximating}
Mark~Alexander Burgess and Archie~C Chapman.
\newblock Approximating the shapley value using stratified empirical bernstein sampling.
\newblock In \emph{IJCAI}, pp.\  73--81, 2021.

\bibitem[Chen et~al.(2020)Chen, Si, Li, Chelba, Kumar, Boning, and Hsieh]{chen2020multi}
Hongge Chen, Si~Si, Yang Li, Ciprian Chelba, Sanjiv Kumar, Duane Boning, and Cho-Jui Hsieh.
\newblock Multi-stage influence function.
\newblock \emph{Advances in Neural Information Processing Systems}, 33:\penalty0 12732--12742, 2020.

\bibitem[Chen et~al.(2021)Chen, Li, Yu, Wu, and Miao]{chen2021hydra}
Yuanyuan Chen, Boyang Li, Han Yu, Pengcheng Wu, and Chunyan Miao.
\newblock Hydra: Hypergradient data relevance analysis for interpreting deep neural networks.
\newblock In \emph{Proceedings of the AAAI Conference on Artificial Intelligence}, volume~35, pp.\  7081--7089, 2021.

\bibitem[Choe et~al.(2024)Choe, Ahn, Bae, Zhao, Kang, Chung, Pratapa, Neiswanger, Strubell, Mitamura, et~al.]{choe2024your}
Sang~Keun Choe, Hwijeen Ahn, Juhan Bae, Kewen Zhao, Minsoo Kang, Youngseog Chung, Adithya Pratapa, Willie Neiswanger, Emma Strubell, Teruko Mitamura, et~al.
\newblock What is your data worth to gpt? llm-scale data valuation with influence functions.
\newblock \emph{arXiv preprint arXiv:2405.13954}, 2024.

\bibitem[Cook \& Weisberg(1980)Cook and Weisberg]{cook1980characterizations}
R~Dennis Cook and Sanford Weisberg.
\newblock Characterizations of an empirical influence function for detecting influential cases in regression.
\newblock \emph{Technometrics}, 22\penalty0 (4):\penalty0 495--508, 1980.

\bibitem[Covert et~al.(2024)Covert, Kim, Lee, Zou, and Hashimoto]{covert2024stochastic}
Ian Covert, Chanwoo Kim, Su-In Lee, James Zou, and Tatsunori Hashimoto.
\newblock Stochastic amortization: A unified approach to accelerate feature and data attribution.
\newblock \emph{arXiv preprint arXiv:2401.15866}, 2024.

\bibitem[Deng \& Ma(2023)Deng and Ma]{deng2023computational}
Junwei Deng and Jiaqi Ma.
\newblock Computational copyright: Towards a royalty model for ai music generation platforms.
\newblock \emph{arXiv preprint arXiv:2312.06646}, 2023.

\bibitem[Deng et~al.(2024)Deng, Li, Zhang, and Ma]{deng2024efficient}
Junwei Deng, Ting-Wei Li, Shichang Zhang, and Jiaqi Ma.
\newblock Efficient ensembles improve training data attribution.
\newblock \emph{arXiv preprint arXiv:2405.17293}, 2024.

\bibitem[Ding et~al.(2024)Ding, Ji, Wang, and Xu]{ding2024understanding}
Meng Ding, Kaiyi Ji, Di~Wang, and Jinhui Xu.
\newblock Understanding forgetting in continual learning with linear regression.
\newblock In \emph{Forty-first International Conference on Machine Learning}, 2024.

\bibitem[Dwork et~al.(2006)Dwork, McSherry, Nissim, and Smith]{dwork2006calibrating}
Cynthia Dwork, Frank McSherry, Kobbi Nissim, and Adam Smith.
\newblock Calibrating noise to sensitivity in private data analysis.
\newblock In \emph{Theory of cryptography conference}, pp.\  265--284. Springer, 2006.

\bibitem[Epifano et~al.(2023)Epifano, Ramachandran, Masino, and Rasool]{epifano2023revisiting}
Jacob~R Epifano, Ravi~P Ramachandran, Aaron~J Masino, and Ghulam Rasool.
\newblock Revisiting the fragility of influence functions.
\newblock \emph{Neural Networks}, 162:\penalty0 581--588, 2023.

\bibitem[Evron et~al.(2022)Evron, Moroshko, Ward, Srebro, and Soudry]{evron2022catastrophic}
Itay Evron, Edward Moroshko, Rachel Ward, Nathan Srebro, and Daniel Soudry.
\newblock How catastrophic can catastrophic forgetting be in linear regression?
\newblock In \emph{Conference on Learning Theory}, pp.\  4028--4079. PMLR, 2022.

\bibitem[Fan et~al.(2024)Fan, Pagliardini, and Jaggi]{fan2024doge}
Simin Fan, Matteo Pagliardini, and Martin Jaggi.
\newblock Doge: Domain reweighting with generalization estimation.
\newblock In \emph{Forty-first International Conference on Machine Learning}, 2024.

\bibitem[Feldman \& Zhang(2020)Feldman and Zhang]{feldman2020neural}
Vitaly Feldman and Chiyuan Zhang.
\newblock What neural networks memorize and why: Discovering the long tail via influence estimation.
\newblock \emph{Advances in Neural Information Processing Systems}, 33:\penalty0 2881--2891, 2020.

\bibitem[Gao et~al.(2020)Gao, Biderman, Black, Golding, Hoppe, Foster, Phang, He, Thite, Nabeshima, et~al.]{gao2020pile}
Leo Gao, Stella Biderman, Sid Black, Laurence Golding, Travis Hoppe, Charles Foster, Jason Phang, Horace He, Anish Thite, Noa Nabeshima, et~al.
\newblock The pile: An 800gb dataset of diverse text for language modeling.
\newblock \emph{arXiv preprint arXiv:2101.00027}, 2020.

\bibitem[Ghorbani \& Zou(2019)Ghorbani and Zou]{ghorbani2019data}
Amirata Ghorbani and James Zou.
\newblock Data shapley: Equitable valuation of data for machine learning.
\newblock In \emph{International Conference on Machine Learning}, pp.\  2242--2251. PMLR, 2019.

\bibitem[Ghorbani et~al.(2020)Ghorbani, Kim, and Zou]{ghorbani2020distributional}
Amirata Ghorbani, Michael Kim, and James Zou.
\newblock A distributional framework for data valuation.
\newblock In \emph{International Conference on Machine Learning}, pp.\  3535--3544. PMLR, 2020.

\bibitem[Grosse et~al.(2023)Grosse, Bae, Anil, Elhage, Tamkin, Tajdini, Steiner, Li, Durmus, Perez, et~al.]{grosse2023studying}
Roger Grosse, Juhan Bae, Cem Anil, Nelson Elhage, Alex Tamkin, Amirhossein Tajdini, Benoit Steiner, Dustin Li, Esin Durmus, Ethan Perez, et~al.
\newblock Studying large language model generalization with influence functions.
\newblock \emph{arXiv preprint arXiv:2308.03296}, 2023.

\bibitem[Guo et~al.(2021)Guo, Rajani, Hase, Bansal, and Xiong]{guo2021fastif}
Han Guo, Nazneen Rajani, Peter Hase, Mohit Bansal, and Caiming Xiong.
\newblock Fastif: Scalable influence functions for efficient model interpretation and debugging.
\newblock In \emph{Proceedings of the 2021 Conference on Empirical Methods in Natural Language Processing}, pp.\  10333--10350, 2021.

\bibitem[Hanin \& Rolnick(2018)Hanin and Rolnick]{hanin2018start}
Boris Hanin and David Rolnick.
\newblock How to start training: The effect of initialization and architecture.
\newblock \emph{Advances in neural information processing systems}, 31, 2018.

\bibitem[Hara et~al.(2019)Hara, Nitanda, and Maehara]{hara2019data}
Satoshi Hara, Atsushi Nitanda, and Takanori Maehara.
\newblock Data cleansing for models trained with sgd.
\newblock \emph{Advances in Neural Information Processing Systems}, 32, 2019.

\bibitem[He et~al.(2015)He, Zhang, Ren, and Sun]{he2015delving}
Kaiming He, Xiangyu Zhang, Shaoqing Ren, and Jian Sun.
\newblock Delving deep into rectifiers: Surpassing human-level performance on imagenet classification.
\newblock In \emph{Proceedings of the IEEE international conference on computer vision}, pp.\  1026--1034, 2015.

\bibitem[Ill{\'e}s \& Ker{\'e}nyi(2019)Ill{\'e}s and Ker{\'e}nyi]{illes2019estimation}
Ferenc Ill{\'e}s and P{\'e}ter Ker{\'e}nyi.
\newblock Estimation of the shapley value by ergodic sampling.
\newblock \emph{arXiv preprint arXiv:1906.05224}, 2019.

\bibitem[Ilyas et~al.(2022)Ilyas, Park, Engstrom, Leclerc, and Madry]{ilyas2022datamodels}
Andrew Ilyas, Sung~Min Park, Logan Engstrom, Guillaume Leclerc, and Aleksander Madry.
\newblock Datamodels: Predicting predictions from training data.
\newblock \emph{arXiv preprint arXiv:2202.00622}, 2022.

\bibitem[Jia et~al.(2019{\natexlab{a}})Jia, Dao, Wang, Hubis, Gurel, Li, Zhang, Spanos, and Song]{jia2019efficient}
Ruoxi Jia, David Dao, Boxin Wang, Frances~Ann Hubis, Nezihe~Merve Gurel, Bo~Li, Ce~Zhang, Costas~J Spanos, and Dawn Song.
\newblock Efficient task-specific data valuation for nearest neighbor algorithms.
\newblock \emph{Proceedings of the VLDB Endowment}, 2019{\natexlab{a}}.

\bibitem[Jia et~al.(2019{\natexlab{b}})Jia, Dao, Wang, Hubis, Hynes, G{\"u}rel, Li, Zhang, Song, and Spanos]{jia2019towards}
Ruoxi Jia, David Dao, Boxin Wang, Frances~Ann Hubis, Nick Hynes, Nezihe~Merve G{\"u}rel, Bo~Li, Ce~Zhang, Dawn Song, and Costas~J Spanos.
\newblock Towards efficient data valuation based on the shapley value.
\newblock In \emph{The 22nd International Conference on Artificial Intelligence and Statistics}, pp.\  1167--1176. PMLR, 2019{\natexlab{b}}.

\bibitem[Jiang et~al.(2023)Jiang, Liang, Zou, and Kwon]{jiang2023opendataval}
Kevin Jiang, Weixin Liang, James~Y Zou, and Yongchan Kwon.
\newblock Opendataval: a unified benchmark for data valuation.
\newblock \emph{Advances in Neural Information Processing Systems}, 36, 2023.

\bibitem[Just et~al.(2022)Just, Kang, Wang, Zeng, Ko, Jin, and Jia]{just2022lava}
Hoang~Anh Just, Feiyang Kang, Tianhao Wang, Yi~Zeng, Myeongseob Ko, Ming Jin, and Ruoxi Jia.
\newblock Lava: Data valuation without pre-specified learning algorithms.
\newblock In \emph{The Eleventh International Conference on Learning Representations}, 2022.

\bibitem[Koh \& Liang(2017)Koh and Liang]{koh2017understanding}
Pang~Wei Koh and Percy Liang.
\newblock Understanding black-box predictions via influence functions.
\newblock In \emph{International Conference on Machine Learning}, pp.\  1885--1894. PMLR, 2017.

\bibitem[Kwon \& Zou(2022)Kwon and Zou]{kwon2022beta}
Yongchan Kwon and James Zou.
\newblock Beta shapley: a unified and noise-reduced data valuation framework for machine learning.
\newblock In \emph{International Conference on Artificial Intelligence and Statistics}, pp.\  8780--8802. PMLR, 2022.

\bibitem[Kwon \& Zou(2023)Kwon and Zou]{kwon2023data}
Yongchan Kwon and James Zou.
\newblock Data-oob: Out-of-bag estimate as a simple and efficient data value.
\newblock \emph{ICML}, 2023.

\bibitem[Kwon et~al.(2021)Kwon, Rivas, and Zou]{kwon2021efficient}
Yongchan Kwon, Manuel~A Rivas, and James Zou.
\newblock Efficient computation and analysis of distributional shapley values.
\newblock In \emph{International Conference on Artificial Intelligence and Statistics}, pp.\  793--801. PMLR, 2021.

\bibitem[Kwon et~al.(2023)Kwon, Wu, Wu, and Zou]{kwon2023datainf}
Yongchan Kwon, Eric Wu, Kevin Wu, and James Zou.
\newblock Datainf: Efficiently estimating data influence in lora-tuned llms and diffusion models.
\newblock In \emph{The Twelfth International Conference on Learning Representations}, 2023.

\bibitem[LeCun et~al.(1989)LeCun, Boser, Denker, Henderson, Howard, Hubbard, and Jackel]{lecun1989handwritten}
Yann LeCun, Bernhard Boser, John Denker, Donnie Henderson, Richard Howard, Wayne Hubbard, and Lawrence Jackel.
\newblock Handwritten digit recognition with a back-propagation network.
\newblock \emph{Advances in neural information processing systems}, 2, 1989.

\bibitem[Li \& Yu(2023)Li and Yu]{li2023faster}
Weida Li and Yaoliang Yu.
\newblock Faster approximation of probabilistic and distributional values via least squares.
\newblock In \emph{The Twelfth International Conference on Learning Representations}, 2023.

\bibitem[Li \& Yu(2024)Li and Yu]{li2024robust}
Weida Li and Yaoliang Yu.
\newblock Robust data valuation with weighted banzhaf values.
\newblock \emph{Advances in Neural Information Processing Systems}, 36, 2024.

\bibitem[Lin et~al.(2022)Lin, Zhang, L{\'e}cuyer, Li, Panda, and Sen]{lin2022measuring}
Jinkun Lin, Anqi Zhang, Mathias L{\'e}cuyer, Jinyang Li, Aurojit Panda, and Siddhartha Sen.
\newblock Measuring the effect of training data on deep learning predictions via randomized experiments.
\newblock In \emph{International Conference on Machine Learning}, pp.\  13468--13504. PMLR, 2022.

\bibitem[Martens(2020)]{martens2020new}
James Martens.
\newblock New insights and perspectives on the natural gradient method.
\newblock \emph{Journal of Machine Learning Research}, 21\penalty0 (146):\penalty0 1--76, 2020.

\bibitem[Merity et~al.(2016)Merity, Xiong, Bradbury, and Socher]{merity2016pointer}
Stephen Merity, Caiming Xiong, James Bradbury, and Richard Socher.
\newblock Pointer sentinel mixture models.
\newblock \emph{arXiv preprint arXiv:1609.07843}, 2016.

\bibitem[Mitchell et~al.(2022)Mitchell, Cooper, Frank, and Holmes]{mitchell2022sampling}
Rory Mitchell, Joshua Cooper, Eibe Frank, and Geoffrey Holmes.
\newblock Sampling permutations for shapley value estimation.
\newblock 2022.

\bibitem[Nguyen et~al.(2024)Nguyen, Seo, and Oh]{nguyen2024bayesian}
Elisa Nguyen, Minjoon Seo, and Seong~Joon Oh.
\newblock A bayesian approach to analysing training data attribution in deep learning.
\newblock \emph{Advances in Neural Information Processing Systems}, 36, 2024.

\bibitem[Nohyun et~al.(2022)Nohyun, Choi, and Chung]{nohyun2022data}
Ki~Nohyun, Hoyong Choi, and Hye~Won Chung.
\newblock Data valuation without training of a model.
\newblock In \emph{The Eleventh International Conference on Learning Representations}, 2022.

\bibitem[Okhrati \& Lipani(2021)Okhrati and Lipani]{okhrati2021multilinear}
Ramin Okhrati and Aldo Lipani.
\newblock A multilinear sampling algorithm to estimate shapley values.
\newblock In \emph{2020 25th International Conference on Pattern Recognition (ICPR)}, pp.\  7992--7999. IEEE, 2021.

\bibitem[Park et~al.(2023)Park, Georgiev, Ilyas, Leclerc, and Madry]{park2023trak}
Sung~Min Park, Kristian Georgiev, Andrew Ilyas, Guillaume Leclerc, and Aleksander Madry.
\newblock Trak: attributing model behavior at scale.
\newblock In \emph{Proceedings of the 40th International Conference on Machine Learning}, pp.\  27074--27113, 2023.

\bibitem[Pruthi et~al.(2020)Pruthi, Liu, Kale, and Sundararajan]{pruthi2020estimating}
Garima Pruthi, Frederick Liu, Satyen Kale, and Mukund Sundararajan.
\newblock Estimating training data influence by tracing gradient descent.
\newblock \emph{Advances in Neural Information Processing Systems}, 33:\penalty0 19920--19930, 2020.

\bibitem[Schioppa et~al.(2022)Schioppa, Zablotskaia, Vilar, and Sokolov]{schioppa2022scaling}
Andrea Schioppa, Polina Zablotskaia, David Vilar, and Artem Sokolov.
\newblock Scaling up influence functions.
\newblock In \emph{Proceedings of the AAAI Conference on Artificial Intelligence}, volume~36, pp.\  8179--8186, 2022.

\bibitem[Schraudolph(2002)]{schraudolph2002fast}
Nicol~N Schraudolph.
\newblock Fast curvature matrix-vector products for second-order gradient descent.
\newblock \emph{Neural computation}, 14\penalty0 (7):\penalty0 1723--1738, 2002.

\bibitem[Shapley(1953)]{shapley1953value}
Lloyd~S Shapley.
\newblock A value for n-person games.
\newblock \emph{Contributions to the Theory of Games}, 2\penalty0 (28):\penalty0 307--317, 1953.

\bibitem[Sim et~al.(2022)Sim, Xu, and Low]{sim2022data}
Rachael Hwee~Ling Sim, Xinyi Xu, and Bryan Kian~Hsiang Low.
\newblock Data valuation in machine learning:“ingredients”, strategies, and open challenges.
\newblock In \emph{Proc. IJCAI}, 2022.

\bibitem[S{\o}gaard et~al.(2021)]{sogaard2021revisiting}
Anders S{\o}gaard et~al.
\newblock Revisiting methods for finding influential examples.
\newblock \emph{arXiv preprint arXiv:2111.04683}, 2021.

\bibitem[Tay et~al.(2022)Tay, Xu, Foo, and Low]{tay2022incentivizing}
Sebastian~Shenghong Tay, Xinyi Xu, Chuan~Sheng Foo, and Bryan Kian~Hsiang Low.
\newblock Incentivizing collaboration in machine learning via synthetic data rewards.
\newblock In \emph{Proceedings of the AAAI Conference on Artificial Intelligence}, volume~36, pp.\  9448--9456, 2022.

\bibitem[Tian et~al.(2022)Tian, Liu, Li, Cao, Jia, and Ren]{tian2022private}
Zhihua Tian, Jian Liu, Jingyu Li, Xinle Cao, Ruoxi Jia, and Kui Ren.
\newblock Private data valuation and fair payment in data marketplaces.
\newblock \emph{arXiv preprint arXiv:2210.08723}, 2022.

\bibitem[Vaswani(2017)]{vaswani2017attention}
A~Vaswani.
\newblock Attention is all you need.
\newblock \emph{Advances in Neural Information Processing Systems}, 2017.

\bibitem[Wang \& Jia(2023{\natexlab{a}})Wang and Jia]{wang2023data}
Jiachen~T Wang and Ruoxi Jia.
\newblock Data banzhaf: A robust data valuation framework for machine learning.
\newblock In \emph{International Conference on Artificial Intelligence and Statistics}, pp.\  6388--6421. PMLR, 2023{\natexlab{a}}.

\bibitem[Wang \& Jia(2023{\natexlab{b}})Wang and Jia]{wang2023notegroup}
Jiachen~T Wang and Ruoxi Jia.
\newblock A note on" towards efficient data valuation based on the shapley value''.
\newblock \emph{arXiv preprint arXiv:2302.11431}, 2023{\natexlab{b}}.

\bibitem[Wang \& Jia(2023{\natexlab{c}})Wang and Jia]{wang2023noteknn}
Jiachen~T Wang and Ruoxi Jia.
\newblock A note on" efficient task-specific data valuation for nearest neighbor algorithms".
\newblock \emph{arXiv preprint arXiv:2304.04258}, 2023{\natexlab{c}}.

\bibitem[Wang et~al.(2023)Wang, Zhu, Wang, Jia, and Mittal]{wang2023threshold}
Jiachen~T Wang, Yuqing Zhu, Yu-Xiang Wang, Ruoxi Jia, and Prateek Mittal.
\newblock Threshold knn-shapley: A linear-time and privacy-friendly approach to data valuation.
\newblock \emph{arXiv preprint arXiv:2308.15709}, 2023.

\bibitem[Wang et~al.(2024{\natexlab{a}})Wang, Deng, Chiba-Okabe, Barak, and Su]{wang2024economic}
Jiachen~T Wang, Zhun Deng, Hiroaki Chiba-Okabe, Boaz Barak, and Weijie~J Su.
\newblock An economic solution to copyright challenges of generative ai.
\newblock Technical report, 2024{\natexlab{a}}.

\bibitem[Wang et~al.(2024{\natexlab{b}})Wang, Mittal, and Jia]{wang2024efficient}
Jiachen~T Wang, Prateek Mittal, and Ruoxi Jia.
\newblock Efficient data shapley for weighted nearest neighbor algorithms.
\newblock \emph{arXiv preprint arXiv:2401.11103}, 2024{\natexlab{b}}.

\bibitem[Wang et~al.(2024{\natexlab{c}})Wang, Mittal, Song, and Jia]{wang2024data}
Jiachen~T Wang, Prateek Mittal, Dawn Song, and Ruoxi Jia.
\newblock Data shapley in one training run.
\newblock \emph{arXiv preprint arXiv:2406.11011}, 2024{\natexlab{c}}.

\bibitem[Wu et~al.(2022{\natexlab{a}})Wu, Zou, Braverman, Gu, and Kakade]{wu2022power}
Jingfeng Wu, Difan Zou, Vladimir Braverman, Quanquan Gu, and Sham Kakade.
\newblock The power and limitation of pretraining-finetuning for linear regression under covariate shift.
\newblock \emph{Advances in Neural Information Processing Systems}, 35:\penalty0 33041--33053, 2022{\natexlab{a}}.

\bibitem[Wu et~al.(2024)Wu, Zou, Chen, Braverman, Gu, and Bartlett]{wu2024many}
Jingfeng Wu, Difan Zou, Zixiang Chen, Vladimir Braverman, Quanquan Gu, and Peter Bartlett.
\newblock How many pretraining tasks are needed for in-context learning of linear regression?
\newblock In \emph{The Twelfth International Conference on Learning Representations}, 2024.

\bibitem[Wu et~al.(2022{\natexlab{b}})Wu, Shu, and Low]{wu2022davinz}
Zhaoxuan Wu, Yao Shu, and Bryan Kian~Hsiang Low.
\newblock Davinz: Data valuation using deep neural networks at initialization.
\newblock In \emph{International Conference on Machine Learning}, pp.\  24150--24176. PMLR, 2022{\natexlab{b}}.

\bibitem[Xu et~al.(2021)Xu, Wu, Foo, and Low]{xu2021validation}
Xinyi Xu, Zhaoxuan Wu, Chuan~Sheng Foo, and Bryan Kian~Hsiang Low.
\newblock Validation free and replication robust volume-based data valuation.
\newblock \emph{Advances in Neural Information Processing Systems}, 34:\penalty0 10837--10848, 2021.

\bibitem[Yang et~al.(2023)Yang, Deng, Liu, Huang, Zou, and Li]{yang2023gmvaluator}
Jiaxi Yang, Wenglong Deng, Benlin Liu, Yangsibo Huang, James Zou, and Xiaoxiao Li.
\newblock Gmvaluator: Similarity-based data valuation for generative models.
\newblock \emph{arXiv preprint arXiv:2304.10701}, 2023.

\bibitem[Yang et~al.(2024)Yang, Yue, Chen, and Liu]{yang2024inflation}
Ziao Yang, Han Yue, Jian Chen, and Hongfu Liu.
\newblock On the inflation of knn-shapley value.
\newblock \emph{arXiv preprint arXiv:2405.17489}, 2024.

\bibitem[Zhang et~al.(2024)Zhang, Shen, Xiong, and Kwon]{zhang2024timeinf}
Yizi Zhang, Jingyan Shen, Xiaoxue Xiong, and Yongchan Kwon.
\newblock Timeinf: Time series data contribution via influence functions.
\newblock \emph{arXiv preprint arXiv:2407.15247}, 2024.

\bibitem[Ziegel(2003)]{ziegel2003elements}
Eric~R Ziegel.
\newblock The elements of statistical learning, 2003.

\bibitem[Zou et~al.(2021)Zou, Wu, Braverman, Gu, and Kakade]{zou2021benign}
Difan Zou, Jingfeng Wu, Vladimir Braverman, Quanquan Gu, and Sham Kakade.
\newblock Benign overfitting of constant-stepsize sgd for linear regression.
\newblock In \emph{Conference on Learning Theory}, pp.\  4633--4635. PMLR, 2021.

\end{thebibliography}
\bibliographystyle{iclr2025_conference}

\clearpage

\appendix

\section{Extended Related Works}
\label{appendix:related-work}
\subsection{LOO Influence vs LOOCV}

\rebuttal{It is important to distinguish our LOO influence measure from traditional Leave-One-Out Cross-Validation (LOOCV) \citep{ziegel2003elements}. While both involve removing individual data points, they serve different purposes and yield different interpretations. LOOCV is a model evaluation technique that estimates generalization performance by averaging prediction errors on held-out examples, where smaller errors indicate better model performance. In contrast, LOO influence measures how removing a specific training point affects the model's behavior on validation data, quantifying each training example's importance to the learning process. While LOOCV requires training $N$ separate models to evaluate generalization (where $N$ is the dataset size), LOO influence focuses on understanding the counterfactual impact of individual training points on model behavior. This distinction is crucial as we aim to understand data importance rather than model performance.}

\subsection{Influence Function and Friends}
Influence function \citep{koh2017understanding} has emerged as an important tool for interpreting and analyzing machine learning models. As the influence function requires computing the Hessian inverse, many subsequent works are focusing on improving the scalability of the influence function for large-scale models \citep{guo2021fastif,schioppa2022scaling, grosse2023studying}. More recently, \cite{kwon2023datainf} developed an efficient influence function approximation algorithm that is suitable for LoRA fine-tuning, and \cite{zhang2024timeinf} extends the influence function to time-series datasets. 
In a similar spirit to us, \cite{chen2020multi} a multi-stage extension of influence function to trace a fine-tuned model’s behavior back to the pretraining data. However, they changed the original loss function and added a regularization term to account for intermediate checkpoints. 
The most closely related to our work is \citep{choe2024your}. Similar to us, they also make use of the low-rank gradient decomposition and random projection to enable efficient computation and storage of per-sample gradient. However, their approach still requires computing per-sample gradient vectors for all training data on the final model checkpoint, which is effectively equivalent to one model retraining and takes a significantly longer time than data value embedding.

Besides influence function-based approaches, \cite{park2023trak} proposed \emph{TRAK}, which assumes the model linearity and derived a closed-form expression by analyzing one Newton-step on the optimal model parameter on the leave-one-out dataset for logistic regression. However, it generally requires aggregating the estimator across multiple trained models for a reasonable performance and thus difficult to scale to large-scale models. Another closely related literature to this work is TracIN \citep{pruthi2020estimating}, which estimates the influence of each training data by exploiting the gradient over all iterations.

\subsection{Data Shapley and Friends}
\label{appendix:related-work-shapley}

\emph{Data Shapley} is one of the first principled approaches to data attribution being proposed \cite{ghorbani2019data, jia2019towards}. Data Shapley is based on the famous \emph{Shapley value} \citep{shapley1953value}. Since its introduction in 2019 \citep{ghorbani2019data, jia2019towards}, Data Shapley has rapidly gained popularity as a principled solution for data attribution. Due to the computationally expensive nature of retraining-based Data Shapley, various Monte Carlo-based approximation algorithms have been developed \citep{jia2019towards, illes2019estimation, okhrati2021multilinear, burgess2021approximating, mitchell2022sampling, lin2022measuring, wang2023notegroup, li2023faster, covert2024stochastic}, these methods still necessitate extensive computational resources due to repeated model retraining, which is clearly impractical for modern-sized ML models. 
Many of its variants have been proposed. 
\cite{kwon2022beta} argues that the efficiency axiom is not necessary for many machine learning applications, and the framework of \emph{semivalue} is derived by relaxing the efficiency axiom. \cite{lin2022measuring} provide an alternative justification for semivalue based on causal inference and randomized experiments. 
Based on the framework of semivalue, \cite{kwon2022beta} propose \emph{Beta Shapley}, which is a collection of semivalues that enjoy certain mathematical convenience. \cite{wang2023data} propose \emph{Data Banzhaf}, and show that the Banzhaf value, another famous solution concept from cooperative game theory, achieves more stable valuation results under stochastic learning algorithms. 
\cite{li2024robust} further improves the valuation stability by considering value notions outside the scope of semivalue. 
The classic leave-one-out error is also a semivalue, where the \emph{influence function} \cite{cook1980characterizations, koh2017understanding,  grosse2023studying} is generally considered as its approximation. 
Another line of works focuses on improving the computational efficiency of Data Shapley by considering K nearest neighbor (KNN) as the surrogate learning algorithm for the original, potentially complicated deep learning models \citep{jia2019efficient, wang2023noteknn,wang2023threshold, wang2024efficient, yang2024inflation}. 
\cite{ghorbani2020distributional, kwon2021efficient, li2023faster} consider Distributional Shapley, a generalization of Data Shapley to data distribution. 
Finally, \citet{wang2024data} proposes In-Run Data Shapley, a scalable alternative to the original Data Shapley that avoids the need for repeated retraining. However, a critical limitation of In-Run Data Shapley is its requirement of knowing validation data in advance, as we detailed in Appendix \ref{appendix:limitations}.

\subsection{Alternative Data Attribution Methods}
\label{appendix:related-work-nonshapley}

There have also been many approaches for data attribution that do not belong to the family of influence function or the Shapley value. For a detailed survey, we direct readers to \cite{sim2022data} and \cite{jiang2023opendataval}. 
Here, we summarize a few representative works. 
Datamodel \citep{ilyas2022datamodels} is similar to retraining-based Data Shapley that requires training thousands of models on different data subsets to estimate the data influence of each training datum. It leverages a linear regression model to predict the model performance based on the input training set, and uses the learned regression coefficient as the measure of data influence. \cite{xu2021validation} proposed a diversity measure known as robust volume (RV) for appraising data sources. \cite{tay2022incentivizing} devised a valuation method leveraging the maximum mean discrepancy (MMD) between the data source and the actual data distribution. \cite{nohyun2022data} introduced a \emph{complexity-gap score} for evaluating data value without training, specifically in the context of overparameterized neural networks.
\cite{wu2022davinz} applied a domain-aware generalization bound based on neural tangent kernel (NTK) theory for data valuation. \cite{amiri2022fundamentals} assessed data value by measuring statistical differences between the source data and a baseline dataset. 
\cite{just2022lava} utilized a specialized Wasserstein distance between training and validation sets as the utility function, alongside an efficient approximation of the LOO error. \cite{kwon2023data} utilized random forests as proxy models to propose an efficient, validation-free data valuation algorithm. \citet{nguyen2024bayesian} takes a Bayesian view of data attribution and is able to evaluate the variance of LOO. 
Similarity-based data attribution techniques evaluate the contribution of individual data points \citep{yang2023gmvaluator} by measuring their resemblance to other points in the dataset or model outputs. However, while being highly scalable, these works often lack a formal theoretical justification as influence function or Data Shapley-based approaches.

\clearpage

\clearpage

\section{Limitations of the Existing Data Attribution Techniques for Foundation Models}
\label{appendix:limitations}

\subsection{Influence Function}

Influence functions \citep{cook1980characterizations, koh2017understanding} are a classical technique from robust statistics, adapted for machine learning to measure how the removal of a single data point affects the performance of a trained model. Influence functions quantify the sensitivity of a model's predictions to specific data points, offering insights into the importance of individual training samples. In the machine learning framework, they are particularly useful for diagnosing model behavior, understanding dataset quality, and identifying mislabeled or harmful data points. The core idea of influence functions is to approximate the effect of removing a data point from the training set without needing to retrain the model. Instead of actually excluding a point and retraining, influence functions leverage the model’s final parameters and compute the impact of a point's removal based on the gradient and Hessian inverse of the loss function at the final model state. 
Formally, the influence of a training data point $ z_i $ on the loss $ \ell(\theta, z^{(val)}) $ at a validation point $ z^{(val)} $ is defined as:
\begin{align*}
\infl(z_i) := -\nabla_\theta \ell(\theta, z^{(val)})^\top \hessian^{-1} \nabla_\theta \ell(\theta, z_i)
\end{align*}
where $\theta$ is the final model parameter after training, $\hessian = \frac{1}{N} \sum_{i=1}^N \nabla_\theta^2 \ell(\theta, z_i)$ is the Hessian of the total training loss at $\theta $, $ \nabla_\theta \ell(\theta, z^{(val)})$ and $\nabla_\theta \ell(\theta, z_i)$ are the gradients of the loss at the validation point and the training point, respectively.

\textbf{Limitation: Neglecting Training Phases and Unrealiable Approximation to LOO.} 
A key limitation of influence function techniques is their exclusive focus on the final model parameters, thereby ignoring the intermediate dynamics of the training process. By assessing data contributions solely based on the final trained model, influence functions fail to capture how each data point influenced the model's updates throughout training. This narrow focus introduces inaccuracies, as it overlooks the cumulative effects of model fluctuations during the training iterations. 
Consequently, influence functions can be less accurate in evaluating data contributions, particularly in large-scale models where the training process plays a significant role. For instance, in modern training paradigms for large language models (LLMs), models are typically pretrained on a broad corpus and subsequently fine-tuned on specialized domains. Influence functions, however, cannot differentiate between the impacts of data points during pretraining and fine-tuning phases. Relying solely on the final model parameters after fine-tuning, they miss how pretraining data contributed to learning general language structures or how fine-tuning data adapted the model to specific domains. This inability to account for different training stages results in incomplete and often noisy estimates of data contributions, thereby reducing the precision of attribution in multi-stage training processes. 

Moreover, our analysis in Section \ref{appendix:inf-approximate-LOO} demonstrates that the influence function approximates the expected data influence across different training trajectories only under overly simplistic conditions, which are often violated in practice. These conditions, such as assuming identical intermediate model checkpoints and Hessian matrices, almost never hold in real-world training scenarios where model evolve significantly. This highlights the inadequacy of influence functions in accurately capturing data contributions, underscoring the necessity for more comprehensive data attribution methods that consider the entire training trajectory.

\textbf{Neglecting Training Phases Necessitates Unreasonable Assumptions and Often Require Model Retraining.} 
Additionally, the focus on the final model necessitates assumptions of convergence and strong convexity to ensure reliable results. In many real-world settings, where models are non-convex and may not fully converge, these assumptions are often violated, leading to further inaccuracies in the data contribution estimates. 
As the influence function score is often found to be highly noisy in practice \citep{basu2020influence,sogaard2021revisiting,bae2022if,epifano2023revisiting}, it typically necessitates multiple model retraining to produce reasonable results \citep{deng2024efficient}, which can undermine their original computational efficiency advantage.

\subsection{In-Run Data Shapley}

In-Run Data Shapley \citep{wang2024data} is a data attribution technique designed to evaluate the contribution of individual data points during a single training run of machine learning models. It builds on the traditional Data Shapley framework, which stems from cooperative game theory. The Shapley value, originally proposed by Lloyd Shapley in 1953, distributes total utility fairly among all contributing players based on their marginal contributions. Applying this concept to machine learning, Data Shapley attributes the contribution of each data point in a training dataset by assessing its influence on model performance. However, standard Data Shapley methods face limitations in scalability because they require numerous retraining iterations on different data subsets. These computational demands make them impractical for large-scale models such as foundation models. 
To address these challenges, In-Run Data Shapley was introduced as a scalable alternative that avoids the need for repeated retraining. Instead, it leverages the iterative nature of model training, specifically neural networks, where parameters are updated in small increments. By tracking gradient updates at each training step, In-Run Data Shapley calculates the contribution of individual data points toward the final model without retraining. It approximates the Shapley value using local utility functions tied to specific gradient updates and extends these to the full training process, capturing cumulative contributions. This method reduces the computational overhead to a level comparable with standard training runs while maintaining the theoretical fairness and interpretability of Shapley values.

\textbf{Limitation: Requirement of Validation Data in Advance.} 
One of the key limitations of In-Run Data Shapley is its reliance on the availability of validation data prior to the start of training. The technique calculates data contribution by examining the impact of training points on model performance as measured against the validation set. Thus, access to this validation data throughout the training process is necessary to compute meaningful Shapley values at each iteration. This restriction can limit the applicability of In-Run Data Shapley in scenarios where validation data is not immediately available, such as in certain real-time learning environments or when the validation set is defined only after training. Potential workarounds, such as saving intermediate model checkpoints to calculate contributions post-training, add complexity to the process and might be unreliable.

\subsection{Similarity-Based Techniques}

Similarity-based data attribution techniques evaluate the contribution of individual data points \citep{yang2023gmvaluator} by measuring their resemblance to other points in the dataset or model outputs. These methods typically calculate distances or similarities between data points using metrics such as Euclidean distance, cosine similarity or learned perceptual features. These methods are computationally efficient compared to more complex attribution approaches like Shapley values or influence functions. Since they do not rely on model retraining or gradient-based analyses, similarity-based techniques can quickly estimate data contribution, making them useful in large-scale datasets or models where computational resources are a concern. 
% For instance, some approaches rank training samples based on their similarity to generated outputs, attributing higher importance to points that are closer in feature space to the generated examples.

\textbf{Limitation: Lack of Formal Theoretical Justification.} 
While similarity-based techniques offer computational advantages, they lack the formal theoretical guarantees provided by methods such as Shapley values or influence functions. These techniques assume that closeness in feature space directly correlates with data contribution, which is not always true, particularly in high-dimensional spaces where distance metrics may not reflect meaningful relationships. Furthermore, these approaches often fail to account for the complex, non-linear interactions between data points and the model’s learning process, resulting in potentially biased or incomplete attributions. Without a formal grounding in cooperative game theory or model-based influence estimation, the results of similarity-based techniques are more heuristic and may not hold across different models or datasets. 
Additionally, because similarity metrics can be sensitive to the chosen feature representation or distance measure, the results can vary significantly depending on these choices. This lack of robustness limits their reliability in critical applications where precise data attribution is required.

\clearpage

\section{Algorithm Details}
\label{appendix:algorithm-details}

\subsection{Derivation Details for Section \ref{sec:method-unrolling}}
\label{appendix:unrolling-sgd}

Suppose $\zstar$ is a data point that participates in the training during the \emph{first} iteration. Denote $\batch_t$ as the training batch in the $t$-th iteration. 
For standard Stochastic Gradient Descent (SGD), we have:
\begin{equation}
    \theta_{k+1} = \theta_k - \eta_k \sum_{z \in \batch_k} \nabla \ell(\theta_k, z)
\end{equation}
for $k = 0, \ldots, T-1$, where $\eta_k$ is the learning rate at iteration $k$.

For validation data $\zval$, we aim to estimate the change in $\ell(\theta_T, \zval)$ by removing $\zstar$ from the first iteration. 
Specifically, we want to estimate $\ell(\theta_T', \zval) - \ell(\theta_T, \zval)$ where:
\begin{equation}
    \theta_{1}' = \theta_0 - \eta_0 \sum_{z \in \batch_0 \setminus \{\zstar\}} \nabla \ell(\theta_0, z)
\end{equation}
and
\begin{equation}
    \theta_{k+1}' = \theta_k' - \eta_k \sum_{z \in \batch_k} \nabla \ell(\theta_k', z)
\end{equation}
for $k = 1, \ldots, T-1$.

To approach this problem, we define an interpolation between $\theta_T$ and $\theta_T'$:
\begin{equation}
    \theta_{1}(\eps) := \theta_0 - \eta_0 \sum_{z \in \batch_0 \setminus \{\zstar\}} \nabla \ell(\theta_0, z) - \eta_0 (1-\eps) \nabla \ell(\theta_0, \zstar)
\end{equation}
where $\theta_{T}(\eps)$ is defined accordingly. Note that $\theta_T(0) = \theta_T$ and $\theta_T(1) = \theta_T'$.

By taking the first-order Taylor expansion at $\eps=0$, we have:
\begin{align}
    \ell(\theta_T', \zval) - \ell(\theta_T, \zval) &= 
    \ell(\theta_T(1), \zval) - \ell(\theta_T(0), \zval) \nonumber \\
    &\approx 
    \left.\frac{\partial}{\partial \eps} \ell(\theta_T(\eps), \zval)\right|_{\eps=0} \nonumber \\
    &= \nabla \ell(\theta_T, \zval)^\top \left.\frac{\partial \theta_T(\eps)}{\partial \eps}\right|_{\eps=0}
\end{align}

Now, we derive $\left.\frac{\partial \theta_T(\eps)}{\partial \eps}\right|_{\eps=0}$ by observing the following recursive relation for all $k \geq 1$:
\begin{align}
    \frac{\partial \theta_{k+1}(\eps)}{\partial \eps} 
    &=
    \frac{\partial \theta_k(\eps)}{\partial \eps} -
    \eta_k \sum_{z \in \batch_k} \nabla^2 \ell(\theta_k(\eps), z) \frac{\partial \theta_k(\eps)}{\partial \eps} \\
    &= 
    \frac{\partial \theta_k(\eps)}{\partial \eps} 
    \left(
    \iden - \eta_k \hessian_k(\eps) 
    \right)
\end{align}
where $\hessian_k(\eps) = \sum_{z \in \batch_k} \nabla^2 \ell(\theta_k(\eps), z)$ is the Hessian and $\iden$ is the identity matrix. 
Additionally, for the first iteration where $\zstar$ participates, we have
\begin{equation}
    \frac{\partial \theta_{1}(\eps)}{\partial \eps} = \eta_0 \nabla \ell(\theta_0, \zstar)
\end{equation}

Expanding the recursion and substituting it back into our original expression, we get:
\begin{align*}
\ell(\theta_T(1), \zval) - \ell(\theta_T(0), \zval) 
&\approx
\frac{\del}{\del \eps} \ell(\theta_T(\eps), \zval)|_{\eps=0} \\
&= \eta_0 \nabla \ell(\theta_T, \zval)^\top \underbrace{\left[ \prod_{k=1}^{T-1} (\iden - \eta_k \hessian_k) \right]}_{\text{cumulative effect}} \nabla \ell(\theta_0, \zstar)
\end{align*}

This final expression gives an estimate of the influence of removing $\zstar$ from the first iteration on the loss on $\zval$ at the final iteration. The term $\prod_{k=1}^{T-1} (\iden - \eta_k \hessian_k)$ represents the \emph{cumulative effect of all training iterations on the initial influence.} This product captures how the impact of the initial change propagates through the entire training process, accounting for the learning rate and the training data at each subsequent step.

\clearpage

\subsection{Error Guarantee for Unrolling-based Approach}
\label{appendix:unroll-error-guarantee}

In this section, we derive the approximation error guarantee of the unrolling differentiation estimator 
\begin{align*} \unrollest := \left.\frac{\partial \theta_T(\eps)}{\partial \eps}\right|{\eps=0} = \eta{\start} \left[ \prod_{k=\start+1}^{T-1} (\iden - \eta_k \hessian_k) \right] \nabla \ell(\theta_{\start}, \zstar), \end{align*} 
for non-convex loss functions. A very loose bound for $\norm{\theta_T - \theta_T' - \unrollest}$ has been derived in \citet{hara2019data}. Here, we improve the error bound by additionally considering the decay of the learning rate and the spectral norm of Hessian matrices as training progresses. Notably, we establish a uniform bound on the gap.

Assume that $\ell(z; \theta)$ is twice differentiable with respect to the parameter $\theta$, and we train the model for $T$ iterations. We make the following assumptions:

\begin{enumerate}
    \item \textbf{Learning Rate Schedule:} The learning rate $\eta_t$ at iteration $t$ follows the schedule $\eta_t = \frac{\eta_{\max}}{\sqrt{t}}$ where $\eta_{\max} = \frac{C}{\sqrt{T}}$ for some constant $C$. 
    \textbf{Justification:} The decaying learning rate schedule $\eta_t = \frac{\eta_{\max}}{\sqrt{t}}$ is a common choice in neural network training in famous literature \citep{vaswani2017attention}. This schedule allows for larger step sizes during the initial phases of training, facilitating rapid convergence, while gradually reducing the step sizes to fine-tune the model parameters and ensure stability as training progresses. The max learning rate $\eta_{\max} = \mO\left(\frac{1}{\sqrt{T}}\right)$ ensures that the cumulative step sizes remain bounded over $T$ iterations, which is crucial for deriving meaningful error bounds. This approach balances the trade-off between exploration and convergence, making it well-suited for training deep neural networks where maintaining stability is essential.
    \item \textbf{Hessian Spectral Norm Decay:} There exists a constant $\Lambda > 0$ such that the Hessian matrices satisfy $    \hessian_t \preceq \frac{\Lambda}{\sqrt{t}} \iden$ for all $t \geq 1$. \textbf{Justification:} The assumption that the spectral norm of the Hessian matrices decays as $\hessian_t \preceq \frac{\Lambda}{\sqrt{t}} \iden$ is grounded in the observation that, as training progresses, the optimization landscape often becomes flatter around minima. This reduction in curvature implies that the Hessian's eigenvalues decrease, leading to smaller spectral norms. Such behavior is typical in many deep learning scenarios where initial training steps navigate regions of high curvature, followed by stabilization in flatter regions as the model converges. Additionally, this assumption aligns with empirical findings in deep learning literature \citep{}, where Hessian's spectral norm has been observed to decrease over time, thereby facilitating more stable and efficient convergence. By incorporating this decay, we account for the diminishing influence of curvature on parameter updates, which is critical for tightening the error bounds in our analysis.
\end{enumerate}

\newcommand{\Diff}{D}
\newcommand{\Zfactor}{Z}
\newcommand{\Ztstar}{\Zfactor^{*}_{s}}
\newcommand{\Zt}{\Zfactor_{s}}
\newcommand{\Dt}{\Diff_{s}}
\newcommand{\Zk}{\Zfactor_{k}}

Under these assumptions, we proceed to derive a uniform bound on the approximation error $\|\theta_T - \theta_T' - \unrollest\|$. This bound provides theoretical guarantees for the effectiveness of the unrolling-based approach in estimating the influence of removing a training data point on the final model parameters. The derivation leverages the decaying learning rate and the diminishing spectral norm of the Hessian matrices to tighten the error bounds compared to previous work \citep{hara2019data}.

\begin{theorem}
\label{thm:unrolling-error-guarantee}
Assume that $\ell(z; \theta)$ is twice differentiable, that the Hessian $\nabla_\theta^2 \ell(z; \theta)$ is $L$-Lipschitz continuous with respect to $\theta$, and that the gradient norm is bounded, i.e., $\| \nabla_\theta \ell(z; \theta) \| \leq G$ for all $z$ and $\theta$. Furthermore, assume that the learning rate $\eta_t$ at iteration $t$ follows the schedule $\eta_t = \frac{\eta_{\max}}{\sqrt{t}}$, where $\eta_{\max} = \frac{C}{\sqrt{T}}$ for some constant $C > 0$. Then, for the unrolling differentiation estimator $\unrollest$, the approximation error satisfies
\begin{align}
\norm{(\theta_T - \theta_T') - \unrollest} \leq \frac{32}{3} G^2 C^3 L e^{C \Lambda}
\label{eq:bound_nonconvex}
\end{align}
\end{theorem}

\begin{proof}
By Cauchy's Mean Value Theorem, for each iteration $s \in \{\start, \ldots, T-1\}$, there exists $r \in [0,1]$ such that for $\theta_{s}^* := r \theta'_s + (1 - r) \theta_s$, we have
\[
\sum_{z \in \batch_s} \left(\nabla_\theta \ell(z; \theta'_s) - \nabla_\theta \ell(z; \theta_s)\right) = \hessian_{s}^{*} (\theta'_s - \theta_s),
\]
where $\hessian_s^* := \sum_{z \in \batch_s} \nabla_\theta^2 \ell(z; \theta_s^*)$. Define $\Zt := (\iden - \eta_s \hessian_s)$ and $\Ztstar := (\iden - \eta_s \hessian_s^*)$. Then, we have
\[
\theta'_{s+1} - \theta_{s+1} = \Zt (\theta'_s - \theta_s) + \eta_s (\hessian_s - \hessian_s^*) (\theta'_s - \theta_s) = \Zt (\theta'_s - \theta_s) + \Dt,
\]
where $\Dt := \eta_s (\hessian_s - \hessian_s^*) (\theta'_s - \theta_s)$. Recursively applying these equalities over $s \in \{\start, \ldots, T-1\}$, we obtain
\[
\theta'_T - \theta_T = \unrollest + \sum_{s=\start}^{T-1} \prod_{k=s+1}^{T-1} \Zk \Dt.
\]
Hence, the approximation error is given by
\[
\| (\theta_T - \theta'_T) - \unrollest \| = \left\| \sum_{s=\start}^{T-1} \prod_{k=s+1}^{T-1} \Zk \Dt \right\|.
\]
To bound this, we proceed as follows. Given the learning rate schedule $\eta_t = \frac{\eta_{\max}}{\sqrt{t}} = \frac{C}{\sqrt{T t}}$, and the assumption that $\hessian_t \preceq \frac{\Lambda}{\sqrt{t}} \iden$, we have
\[
\|\Zk\| = \| \iden - \eta_k \hessian_k \| \leq 1 + \eta_k \frac{\Lambda}{\sqrt{k}} = 1 + \frac{C \Lambda}{k \sqrt{T}}.
\]
For large $T$ and $k \geq s \geq \start \geq 1$, the term $\frac{C \Lambda}{k \sqrt{T}}$ is small. Thus, we can bound the product of the norms as
\[
\prod_{k=s+1}^{T-1} \|\Zk\| \leq \exp\left( \sum_{k=s+1}^{T-1} \frac{C \Lambda}{k \sqrt{T}} \right) \leq \exp\left( \frac{C \Lambda}{\sqrt{T}} \sum_{k=s+1}^{T-1} \frac{1}{k} \right).
\]
Using the harmonic series approximation,
\[
\sum_{k=s+1}^{T-1} \frac{1}{k} \leq \ln\left( \frac{T}{s} \right) \leq \ln(T).
\]
Thus,
\[
\prod_{k=s+1}^{T-1} \|\Zk\| \leq \exp\left( \frac{C \Lambda \ln T}{\sqrt{T}} \right) \leq e^{C \Lambda}.
\]

Therefore, we have
\[
\left\| \sum_{s=\start}^{T-1} \prod_{k=s+1}^{T-1} \Zk \Dt \right\| \leq e^{C \Lambda} \sum_{s=\start}^{T-1} \|\Dt\|.
\]

Next, we bound $\|\Dt\|$:
\[
\|\Dt\| = \left\| \eta_s (\hessian_s - \hessian_s^*) (\theta'_s - \theta_s) \right\| \leq \eta_s \| \hessian_s - \hessian_s^* \| \cdot \| \theta'_s - \theta_s \|.
\]
Since $\nabla_\theta^2 \ell(z; \theta)$ is $L$-Lipschitz continuous with respect to $\theta$, we have
\[
\| \hessian_s - \hessian_s^* \| \leq L \| \theta'_s - \theta_s \|.
\]
Additionally, we have
\[
\| \theta'_s - \theta_s \| \leq 2 \sum_{t=1}^s \eta_t G = 2 G \sum_{t=1}^s \frac{C}{\sqrt{T t}} \leq 4 G C \frac{\sqrt{s}}{\sqrt{T}},
\]
where we used the bound $\sum_{t=1}^s \frac{1}{\sqrt{t}} \leq 2 \sqrt{s}$.

Thus,
\[
\|\Dt\| \leq \eta_s L \cdot \left(4 G C \frac{\sqrt{s}}{\sqrt{T}}\right)^2
= \Gamma \frac{\sqrt{s}}{T^{1.5}}
\]
where $\Gamma = 16G^2 C^3 L$. 

Substituting this bound into the sum, we obtain
\[
\left\| \sum_{s=\start}^{T-1} \prod_{k=s+1}^{T-1} \Zk \Dt \right\| \leq e^{C \Lambda} \sum_{s=\start}^{T-1} \Gamma \frac{\sqrt{s}}{T^{1.5}}.
\]
We now evaluate the summation:
\[
\sum_{s=\start}^{T-1} \frac{\sqrt{s}}{T^{1.5}} \leq \frac{1}{T^{1.5}} \sum_{s=1}^{T} \sqrt{s} \leq \frac{1}{T^{1.5}} \cdot \frac{2}{3} T^{1.5} = \frac{2}{3},
\]
where we used the bound $\sum_{s=1}^T \sqrt{s} \leq \frac{2}{3} T^{1.5}$.

Therefore,
\[
\left\| \sum_{s=\start}^{T-1} \prod_{k=s+1}^{T-1} \Zk \Dt \right\| \leq e^{C \Lambda} \Gamma \cdot \frac{2}{3} = \frac{32}{3} G^2 C^3 L e^{C \Lambda}.
\]
\end{proof}

\subsection{Computing Data Value Embedding Recursively}
\label{appendix:recursive}

\begin{theorem}[Restate for Theorem \ref{thm:recursive-expression}]
Given generalized Gauss-Newton approximation $
\hessian_t 
\approx 
\sum_{z \in \batcht} \nabla \ell(\theta_t, z) \nabla \ell(\theta_t, z)^\top
$, we have 
$$
\embt(\zstar) \approx 
\eta_t \g \ell(\theta_t, \zstar)
- \eta_t 
\sum_{k=t+1}^{T-1} 
\left(
\sum_{z \in \batch_{k}} 
\left( \g \ell(\theta_{k}, z)^\top 
\g \ell(\theta_t, \zstar) \right)
\emb^{(k)} (z) \right)
$$
\end{theorem}
\begin{proof}
\begin{align*}
&\embt(\zstar) \\
&= \eta_t
\left[ \prod_{k=t+1}^{T-1} (\iden - \eta_k \hessian_k) \right] \nabla \ell(\theta_t, \zstar) \\
&= \eta_t
\left[ \prod_{k=t+2}^{T-1} (\iden - \eta_k \hessian_k) \right] 
(\iden - \eta_{t+1} \hessian_{t+1})
\nabla \ell(\theta_t, \zstar) \\
&\approx \eta_t
\left[ \prod_{k=t+2}^{T-1} (\iden - \eta_k \hessian_k) \right] 
\left( \iden - \eta_{t+1} \sum_{z \in \batch_{t+1}} \g \ell(\theta_{t+1}, z) \ell(\theta_{t+1}, z)^\top
\right)
\nabla \ell(\theta_t, \zstar) \\
&= \eta_t
\left[ \prod_{k=t+2}^{T-1} (\iden - \eta_k \hessian_k) \right] \nabla \ell(\theta_t, \zstar)
- \eta_t
\sum_{z \in \batch_{t+1}} 
\left(
\eta_{t+1}
\left[ \prod_{k=t+2}^{T-1} (\iden - \eta_k \hessian_k) \right] 
\g \ell(\theta_{t+1}, z) \right) \g \ell(\theta_{t+1}, z)^\top 
\g \ell(\theta_t, \zstar) \\
&= \eta_t
\left[ \prod_{k=t+2}^{T-1} (\iden - \eta_k \hessian_k) \right] \nabla \ell(\theta_t, \zstar)
- \eta_t
\sum_{z \in \batch_{t+1}} 
\left( \g \ell(\theta_{t+1}, z)^\top 
\g \ell(\theta_t, \zstar) \right)
\emb^{(t+1)} (z) \\
&= 
\eta_t \g \ell(\theta_t, \zstar)
- \eta_t 
\sum_{k=t+1}^{T-1} 
\left(
\sum_{z \in \batch_{k}} 
\left( \g \ell(\theta_{k}, z)^\top 
\g \ell(\theta_t, \zstar) \right)
\emb^{(k)} (z) \right)
\end{align*}
The transition from the penultimate to the final line involves generalizing the summation over $\batch_{t+1}$ to include all batches from $t+1$ to $T-1$, effectively unrolling the recursive computation. In other words, the ``data value embedding'' for data points in $t$th iteration can be approximated by its gradient subtracted by a linear combination of the data value embedding in the later iterations, where the weight of each embedding is determined by the gradient similarity $\g \ell(\theta_{k}, z)^\top 
\g \ell(\theta_t, \zstar)$. 
\end{proof}

\clearpage

\subsection{Generalized Gauss-Newton Approximation to Hessian}
\label{appendix:GGN-approximation}

In this section, we justify the use of Generalized Gauss-Newton (GGN) as the approximation to the Hessian matrix. 
Similar derivation can be found in many literature and textbooks, such as \cite{bartlett1953approximate, schraudolph2002fast}. This approach has also been used in other data attribution and optimization techniques for approximating Hessian matrices \citep{martens2020new, kwon2023datainf, grosse2023studying}. 

The cross-entropy loss function for classification with one-hot encoded labels is defined as:
$$
L(\mathbf{y}, \mathbf{f}) = -\sum_{i=1}^C y_i \log(f_i)
$$
where $\mathbf{y} = [y_1, y_2, \dots, y_C]^\top$ is the one-hot encoded true label vector and $\mathbf{f} = [f_1, f_2, \dots, f_C]^\top$ is the vector of predicted probabilities from the model. In this paper, we restrict our focus to the cross-entropy loss, as it is the most commonly used loss function, and many LLMs are pre-trained with the cross-entropy loss function. 

By chain rule, the derivative of $L$ with respect to $f_i$ is $\frac{\partial L}{\partial f_i} = -\frac{y_i}{f_i}$. Since $\mathbf{y}$ is a one-hot vector, only the correct class $k$ has $y_k = 1$, while all other $y_i = 0$ for $i \ne k$. This simplifies the gradient to:
$$
\nabla_\theta L = -\frac{1}{f_k} \frac{\partial f_k}{\partial \theta}
$$
Thus, the gradient depends only on the derivative of $f_k$ (the predicted probability for the correct class) with respect to $\theta$. 
The Hessian $\hessian$ is the second derivative of the loss with respect to $\theta$:
$$
\hessian = \nabla_\theta^2 L = \frac{\partial}{\partial \theta} \left( \nabla_\theta L \right ) = \frac{\partial}{\partial \theta} \left( -\frac{1}{f_k} \frac{\partial f_k}{\partial \theta} \right )
= \frac{1}{f_k^2} \frac{\partial f_k}{\partial \theta} \left(\frac{\partial f_k}{\partial \theta}\right)^\top - \frac{1}{f_k} \frac{\partial^2 f_k}{\partial \theta^2}
$$
Applying the product rule and assuming the second derivative $\frac{\partial^2 f_k}{\partial \theta^2}$ is negligible (which is common when $f_k$ is approximately linear in $\theta$ near the current parameter values), the Hessian simplifies to:
$$
\hessian \approx \frac{1}{f_k^2} \frac{\partial f_k}{\partial \theta} \left( \frac{\partial f_k}{\partial \theta} \right )^\top
$$
Moreover, this approximation matches the outer product of the gradient of loss with respect to model parameter $\theta$:
$$
\nabla_\theta L \nabla_\theta L^\top = \frac{1}{f_k^2} \frac{\partial f_k}{\partial \theta} \left( \frac{\partial f_k}{\partial \theta} \right )^\top
$$
Therefore, the gradient outer product exactly approximates the Hessian matrix under the assumption that $\frac{\partial^2 f_k}{\partial \theta^2}$ is negligible.

\subsection{Gradient Decomposition Technique}
\label{appendix:ghost}

To mitigate the computational cost from per-sample gradient computation, we leverage a gradient decomposition and take advantage of the computations already performed during backpropagation \citep{wang2024data, choe2024your}. We illustrate this technique with a simple linear layer, where the output is $\bs = \ba \bW$, with $\bW \in \R^{d_1 \times d_2}$ being the weight matrix, $\ba = (\ba^{(1)}, \ldots, \ba^{(B)})^\top$ as the input, and $\bs = (\bs^{(1)}, \ldots, \bs^{(B)})^\top$ representing the pre-activation tensor. 
For non-sequential data, $\ba \in \R^{B \times d_1}, \bs \in \R^{B \times d_2}$. Denote a sample batch as $\batch = \{z_1, \ldots, z_B\}$. By chain rule, we can express the gradient of an individual loss $\ell^{(i)} := \ell(w, z_i)$ with respect to $\bW$ as
\begin{equation}
\begin{aligned}
\frac{\del \ell^{(i)}}{\del \bW} =
\frac{\del \ell^{(i)}}{\del \bs^{(i)}} \outerprod
\frac{\del \bs^{(i)}}{\del \bW}
= \frac{\del \ell^{(i)}}{\del \bs^{(i)}} \outerprod \ba^{(i)} 
= \frac{\del \ell }{\del \bs^{(i)}} \outerprod \ba^{(i)}
\label{eq:linear-decompose-main}
\end{aligned}
\end{equation}
where $\ell := \sum_{j=1}^B \ell^{(j)}$ is the aggregated loss, and the last step is because other data points' losses have no dependency on $\bs_i$. Note that the individual's output gradient $\frac{\partial \ell^{(i)}}{\partial \bs^{(i)}} = \frac{\partial \ell}{\partial \bs^{(i)}}$ is readily available during the backpropagation pass in terms of $\ell$. 
Therefore, this method requires only a single backpropagation on $\ell$, maintaining the training speed equivalent to standard training.

In terms of storage improvement, rather than storing the full gradient vectors $\frac{\partial \ell^{(i)}}{\partial \bW} \in \R^{d_1 \times d_2}$ for each data point $z_i$, we instead store the smaller pair $\left( \ba^{(i)}, \frac{\partial \ell}{\partial \bs^{(i)}} \right) \in \R^{d_1 + d_2}$. This reduces memory requirements from $\mO(pTB)$ to $\mO(\sqrt{p}TB)$ for non-sequential data. For sequential data where $\ba \in \R^{B \times S \times d_1}, \bs \in \R^{B \times S \times d_2}$, if $S^2 > d_1 d_2$, it is more memory-efficient to directly store the per-sample gradient vectors, so the storage requirement remains as $\mO(pTB)$.

\subsection{Random projections for large models}
\label{appendix:projection-error}

For large-scale foundation models with billions of parameters, even the reduced storage of $\mO(\sqrt{p}TB)$ can be substantial. In such cases, we apply random projections to compress the stored information further. We use two projection matrices, $\Pa \in \R^{\projdim \times d_1}$ and $\Ps \in \R^{\projdim \times d_2}$, to project $\ba$ and $\frac{\partial \ell}{\partial \bs}$ to lower dimensional space $\R^{\projdim}$ respectively. The projected gradient can then be reconstructed directly from the projected activations and output derivatives: $\left(\Pa \otimes \Ps\right) \left(\ba \otimes \frac{\partial \ell}{\partial \bs} \right) = \left(\Pa \ba\right) \otimes \left(\Ps \frac{\partial \ell}{\partial \bs}\right)$.

\subsection{Parallelized Extension for Influence Embedding Computation}
\label{appendix:influence-checkpointing}

The backpropagation algorithm introduced in Section \ref{sec:alg-backpropagation} for computing data value embeddings operates with a runtime complexity of $\mO(T)$, as it sequentially computes $\embtstart$ for $\start = T-1, \ldots, 0$. While being significantly more efficient than the influence function, which requires re-computing all training gradients on the final model (see Section \ref{sec:eval-efficiency} and Table \ref{tb:complexity}), it can still be costly for long training periods. 
Here, we present \emph{influence checkpointing} technique, a parallelized extension for Algorithm \ref{alg:full-backpropagate}. This extension reduces computational cost by enabling concurrent computation of embeddings at multiple checkpoints throughout the training process. Besides the computational efficiency benefits, it also enables the study of how the influence of individual data points evolves throughout model training, providing valuable insights into the learning process.

\textbf{Influence Checkpointing.} We pick $K$ evenly spaced training steps $0 < t_1 < t_2 < ... < t_K = T$. We then concurrently execute the backpropagation algorithm for value embedding, initiating from each of these intermediate steps. This process yields data value embeddings for each corresponding intermediate checkpoint $\theta_{t_1}, \ldots, \theta_{t_{K-1}}, \theta_{t_K}$. 
We extend our notation for data value embedding and denote $\embmiddle{\start}{\middlestep}(\zstar)$ as the data value embedding of $\zstar \in \batchtstart$ for the intermediate checkpoint $\theta_{\middlestep}$. Note that $\embtstart = \embmiddle{\start}{T}$ for the final model, and we must have $\start < \middlestep$, as later training batches cannot influence earlier checkpoints. 

Consider initiating the backpropagation algorithm in Section \ref{sec:alg-backpropagation} at step $\middlestep$ and stop at step $\middleprevstep$, we will obtain data value embeddings $\embmiddle{\start}{\middlestep}$ for $\start = \middleprevstep, \ldots, \middlestep-1$. We additionally denote $\KernelMatrix{\ta}{\tb} := \prod_{t=\ta}^{\tb-1} (\iden - \eta_t \hessian_t)$. 
From the definition on Equation (\ref{eq:embedding}), the final data value embedding $\embmiddle{\start}{T}(\zstar)$ can be computed from $\embmiddle{\start}{\middlestep}(\zstar)$ as follows:
\begin{align}
\embmiddle{\start}{T}(\zstar) = \embmiddle{\start}{\middlestep}(\zstar)^\top \KernelMatrix{\middlestep}{T}
\label{eq:emb-parallel}
\end{align}
Hence, to recovery of $\embmiddle{\start}{T}(\zstar)$, we additionally store the matrix $\KernelMatrix{\middleprevstep}{\middlestep}$ between steps $\middleprevstep$ and $\middlestep$ during the backpropagation for each $\middlestep$. Consequently, for any $\start$ such that $t_{\ell_0} \le \start < t_{\ell_0+1}$, we have $\KernelMatrix{\middlestep}{T} = \prod_{\ell=\ell_0+1}^{K} \KernelMatrix{\middleprevstep}{\middlestep}$, allowing us to compute $\embmiddle{\start}{T}(\zstar)$ based on (\ref{eq:emb-parallel}). 
A detailed algorithm pseudocode is provided in Algorithm \ref{alg:inf-checkpointing}. The complexity analysis of this algorithm is the same as the original data value embedding algorithm in Table \ref{tb:complexity}, but the actual runtime is being reduced by a factor of $K$ due to parallelism. 

% A full complexity analysis is deferred to Appendix \ref{appendix:pseudocode}.

\textbf{Data Value Dynamics During Training.} 
In addition to its computational benefits, the influence checkpointing algorithm enables a novel capability: tracking the evolution of data influences throughout the entire model training process. If the intermediate checkpoints $\theta_{t_1}, \ldots, \theta_{t_{K-1}}$ was saved—a common practice in foundation model pretraining—we can analyze how the influence of individual data points changes as training progresses.
Specifically, for any training step $\start$ where $t_{\ell_0} \leq \start < t_{\ell_0+1}$, we can compute the data value embedding $\embmiddle{\start}{t_\kappa}(\zstar)$ for any checkpoint $\kappa \geq \ell_0+1$ as
$
\embmiddle{\start}{t_\kappa}(\zstar) = \embmiddle{\start}{\middlestep}(\zstar)^\top 
\left( \prod_{\ell=\ell_0+1}^{\kappa} \KernelMatrix{\middleprevstep}{\middlestep} \right)
$. 
This formulation allows us to estimate data influence scores not only for the final model checkpoint $\theta_T$ but for any intermediate checkpoints $\theta_{t_\kappa}$. As a result, we gain a more fine-grained and dynamic view of how data influences evolve during training, providing deeper insights into the model's learning behavior over time.
To our knowledge, this is the first data attribution method to offer such a principled and practical framework for studying data influence dynamics throughout the training process. This capability opens up new avenues for understanding and optimizing machine learning model training.

\subsection{Pseudocode}
\label{appendix:pseudocode}

\begin{algorithm}
\caption{Backpropagation for computing data value embedding from the final checkpoint}
\begin{algorithmic}[1]
\Require Training steps $T$, learning rates $\{\eta_t\}_{t=0}^{T-1}$, training data gradients $\{\nabla \ell(\theta_t, z)\}_{t=0,z\in \batcht}^{T-1}$

\State \commentcode{Initialization}
% \For{$z \in \mathcal{B}_{T-1}$}
%     \State $\emb^{(T-1)}(z) \gets \eta_{T-1} \nabla \ell(\theta_{T-1}, z)$
% \EndFor
\State $\embM^{(T-1)} \gets \mathbf{0}$. 

\State 

\State \commentcode{Recursion steps}
\For{$t = T-1$ \textbf{down to} $0$}
    \For{$z \in \batcht$}
        \State $\emb^{(t)}(z) \gets \eta_t \nabla \ell(\theta_t, z) - \eta_t \embM^{(t)} \nabla \ell(\theta_t, z)$
    \EndFor
    \If{$t > 0$}
        \State $\embM^{(t-1)} \gets \embM^{(t)} + \sum_{z \in \batcht} \emb^{(t)}(z) \nabla \ell(\theta_t, z)^\top$
    \EndIf
\EndFor

\State \textbf{return} $\{\emb^{(t)}(z)\}_{t=0,z\in\batcht}^{T-1}$
\end{algorithmic}
\label{alg:full-backpropagate}
\end{algorithm}

\begin{algorithm}
\caption{Parallel Influence Checkpointing for Data Value Embedding}
\begin{algorithmic}[1]
\Require Training steps $T$, number of checkpoints $K$, learning rates $\{\eta_t\}_{t=0}^{T-1}$, loss gradients $\{\nabla \ell(\theta_t, z)\}_{t=0,z\in\batcht}^{T-1}$, Hessians $\{\hessian_t\}_{t=0}^{T-1}$
\Ensure Data value embeddings $\{\emb^{(t)}(z)\}_{t=0,z\in\batcht}^{T-1}$

\State Select $K$ evenly spaced checkpoints $0 = t_0 < t_1 < t_2 < \ldots < t_K = T$

\For{$\ell = 1$ \textbf{to} $K$}
    \State Run \Call{BackpropagateSegment}{$\middleprevstep$, $\middlestep$}
\EndFor

\State 

\State \commentcode{Compute final embeddings}
\For{$\ell = 1$ \textbf{to} $K$}
    \For{$\start = t_{\ell-1}$ \textbf{to} $t_\ell - 1$}
        \For{$z \in \batch_{\start}$}
            \State $\emb^{(\start)}(z) \gets \embmiddle{\start}{t_\ell}(z)^\top \prod_{k=\ell+1}^K \KernelMatrix{t_{k-1}}{t_k}$
        \EndFor
    \EndFor
\EndFor

\State \textbf{return} $\{\emb^{(t)}(z)\}_{t=0,z \in \batcht}^{T-1}$

\State 

\Procedure{BackpropagateSegment}{$\ta$, $\tb$}
    \State Initialize and $\embM^{(\tb-1)}$ as in the original algorithm
    \State $\KernelMatrix{\tb}{\tb} \gets \iden$
    \For{$t = \tb-1$ \textbf{down to} $\ta$}
        \For{$z \in \batcht$}
            \State $\embmiddle{t}{\tb}(z) \gets \eta_t \nabla \ell(\theta_t, z) - \eta_t \embM^{(t)} \nabla \ell(\theta_t, z)$
        \EndFor
        \If{$t > \ta$}
            \State $\embM^{(t-1)} \gets \embM^{(t)} + \sum_{z \in \batcht} \embmiddle{t}{\tb}(z) \nabla \ell(\theta_t, z)^\top$
        \EndIf
        \State $\KernelMatrix{t}{\tb} \gets \KernelMatrix{t+1}{\tb} (\iden - \eta_t \hessian_t)$
    \EndFor
    \State \textbf{return} $\{\embmiddle{t}{\tb}(z)\}_{t=\ta,z\in\batcht}^{\tb-1}$, $\KernelMatrix{\ta}{\tb}$
\EndProcedure

\end{algorithmic}
\label{alg:inf-checkpointing}
\end{algorithm}

\clearpage

\subsection{Complexity Summary}
\label{appendix:method-complexity}

In this section, we compare the storage, memory, and computational efficiency of data value embedding with LoGRA \citep{choe2024your}, the most efficient implementation of the influence function to date. LoGRA is currently the only method that supports real-time, on-demand data influence computation when new test data is introduced. Similar to our algorithm, LoGRA initially computes per-sample training gradients on the final model for \emph{all} training data points $\zstar \in \dataset$, where $\dataset$ represents the dataset. It then stores the \emph{projected} Hessian-adjusted gradients $\hessian_T^{-1} \g \ell(\theta_T, \zstar)$ for each $\zstar$, and also assumes layer gradient independence. 

While the random projection step in LoGRA is akin to our approach, LoGRA's requirement to recompute gradients for all training data on the final model $\theta_T$ is computationally intensive, effectively equivalent to one epoch of model training. In contrast, data value embedding captures the training data gradients during the original training process. As discussed in Appendix \ref{appendix:practical-consideration}, the training and disk storage can be handled asynchronously. This means that the gradient storage step in the data value embedding algorithm does not incur additional efficiency costs.

\begin{table}[h]
\centering
\resizebox{\columnwidth}{!}{\begin{tabular}{@{}cccccc@{}}
\toprule
                     & \multicolumn{3}{c}{Storing $\hessian_T^{-1} \g \ell(\theta_T, \zstar)$ / data value embedding}                                                                                                                                                                        & \multicolumn{2}{c}{Compute Influence (dot-product)}                                                              \\ \midrule
                     & Storage                            & Memory                                                   & FLOPS                                                                                                          & Memory                                           & FLOPS                                           \\
LoGRA                & $\mO(|\dataset| \projdimgrad )$ & $\mO(p)$                                                     & $|\dataset|p + |\dataset|\sqrt{p \projdimgrad}/L$                                                  & $\mO((B_{\text{test}}+B_{\text{train}})\projdimgrad)$ & $\mO(B_{\text{test}} B_{\text{train}} \projdimgrad)$ \\
Data Value Embedding & $\mO(TB\projdimgrad)$                             & $\mO(p) / \mO(B\projdimgrad^{2}/L^2)$* & $TB\sqrt{p \projdimgrad}/L / \mO(TB \projdimgrad^{2}/L)$* & $\mO((B_{\text{test}}+B_{\text{train}})\projdimgrad)$ & $\mO(B_{\text{test}} B_{\text{train}} \projdimgrad)$ \\ \bottomrule
\end{tabular}

}
\vspace{-2mm}
\caption{
Summary of the storage, memory, and FLOPS complexity for LoGRA \citep{choe2024your}, the most efficient implementation of the influence function to date. Here, $p$ denotes the model dimension, $\projdimgrad$ is the projected dimension, $T$ represents the number of training iterations, and $B$ is the batch size. $|\dataset|$ is the dataset size, with the relationship $TB = |\dataset| \times \text{\#epochs}$.
$L$ is the number of layers. $B_{test}$ and $B_{train}$ refer to the test and training batch sizes during influence computation, respectively, which are independent of the batch size $B$ used during model training. 
\textbf{*}Since the data value embedding technique involves two distinct steps for storing relevant information for data attribution (storing gradients during training \& computing/storing data value embeddings after training), we include the complexity for both steps. 
For the gradient storage step, the complexity refers to the \emph{additional} cost beyond regular training.
}
\label{tb:complexity}
\end{table}

\subsection{Practical considerations \& Potential Extensions}
\label{appendix:practical-consideration}

In this section, we discuss some practical extensions and considerations for implementing data value embedding for real-world scenarios. 

\textbf{Optimizing I/O operations for seamless training.} During each training iteration, computed gradient representations need to be transferred from GPU to CPU and then written to disk. To prevent this process from blocking the main training loop, we implement several optimizations:
\textbf{(1) Asynchronous I/O operations}: To avoid the gradient storing process blocking the main training loop, we make GPU operations and GPU-CPU transfers asynchronous by using CUDA streams. This allows GPU computations to continue while data is being transferred. We also offload the disk write process to a separate thread or process, allowing the main training loop to proceed without waiting for disk operations to complete. \textbf{(2) Gradient accumulation}: Instead of writing gradients to disk after every iteration, we can accumulate them over multiple iterations and then write them in bulk. This reduces the frequency of disk I/O operations, improving overall efficiency.

\textbf{Approximating data value embeddings from checkpoints alone.} 
In situations where only intermediate model checkpoints are accessible and per-training-step (projected) gradient vectors are unavailable—such as when modifying the training loop's implementation is impossible or when disk storage is limited—we can adapt our approach by assuming that there is only one gradient update step between each checkpoint, similar to assumptions made in other data attribution literature \citep{pruthi2020estimating}. Under this assumption, we compute the gradient for each training point at the checkpoint immediately following its corresponding training iteration. These estimated gradients are then used to execute the backpropagation algorithm, enabling the computation of data value embeddings without requiring gradient saving during the original training run.

\textbf{Dataset-level attribution through embedding aggregation.} In practical applications, stakeholders often require valuation at the dataset level rather than for individual data points. To address this need, a natural extension of our approach is to compute the data value embedding for a dataset by summing the data value embeddings of all constituent data points from the same source. This method offers a significant advantage over the summation of expected LOO scores, as it inherently accounts for complex inter-data interactions throughout the training process. However, data value embeddings are derived based on first-order Taylor approximations. While these approximations are accurate for estimating small perturbations to the model, making them suitable for predicting the effects of removing individual training points, their accuracy may diminish when aggregating over larger sets of data points. The potential discrepancy between individual-level accuracy and dataset-level aggregation presents an interesting avenue for future research.

\clearpage

\section{Influence Function as an Inadequate Approximation for the Expected Leave-One-Out Score}
\label{appendix:inf-approximate-LOO}

\textbf{Expected LOO.} The expected LOO is an alternative to traditional LOO that has been discussed in the past literature \citep{feldman2020neural}. 
% to address the limitation of the traditional LOO in handling randomness inherent in the training process. Although this issue is orthogonal to the main limitation of order-dependence that we aim to address, we still compare our approach with expected LOO to provide insights into their relationships. 
The \emph{expected LOO} is the trajectory-specific LOO averaged over all possible training runs characterized by different random initializations and mini-batch selections.
Formally, it is defined as
$
\eloo(\zstar; \zval) := \E_{\omega} \left[ \ell(\theta_T'(\omega), \zval) - \ell(\theta_T(\omega), \zval) \right]
$ where $\theta_T(\omega)$ and $\theta_T'(\omega)$ denote the final model parameters obtained with and without $\zstar$, respectively, under the randomness $\omega$ which encodes the choices of training batch order and parameter initialization. 
While the expected LOO offers a general assessment of a data point's influence by averaging over multiple training runs, it may obscure the variability introduced by stochastic training dynamics. In contrast, by accounting for factors such as random initialization and mini-batch selection in a specific run, we argue that the trajectory-specific LOO provides a fine-grained assessment of a data point's impact \emph{for the trained model}. 
% It captures temporal aspects of optimization such as learning rate schedules and early stopping. 
This is particularly important in practical scenarios, such as deploying a specific model for production, where stakeholders are interested in the valuation of data with respect to that specific deployed model rather than the general learning algorithm.

While the influence function provides valuable insights, it overlooks the specific training trajectory. This raises the question: \emph{Can the influence function be interpreted as an estimate of the trajectory-specific leave-one-out score?}

We consider the following training batch sampling process. Let $\sigma := (\batch_0, \ldots, \batch_{T-1})$ represent a fixed sequence of training batches formed from the leave-one-out dataset $\dataset \setminus {\zstar}$. The training point $\zstar$ is uniformly likely to be added to any one of the training batches in $\sigma$. Additionally, denote $\seqtstart := (\batch_0, \ldots, \batch_{\start} \cup \{\zstar\}, \ldots, \batch_{T-1})$ as the training batch sequence where $\zstar$ is incorporated into the $\start$-th batch. Let $\theta_k^{(\seqtstart)}$ denote the model parameters at the $k$-th iteration when training with batch sequence $\seqtstart$, and let $\hessian_k^{(\seqtstart)}$ denote the Hessian matrix at the $k$-th iteration when training on sequence $\seqtstart$. 

When the specific training iteration $\start$ where the training point of interest $\zstar$ is added is unknown, it is natural to estimate the expected influence score across all possible scenarios. The expected influence score for $\zstar$ based on the unrolling differentiation approximation is given by: 
\begin{align} 
\E_{\start \sim [T-1]}
\left[ 
\eta_{\start} \nabla \ell(\theta_T^{(\seqtstart)}, \zval)^\top \left[ \prod_{k=\start+1}^{T-1} (\iden - \eta_k \hessian_k^{(\seqtstart)}) \right] \nabla \ell(\theta_{\start}^{(\seqtstart)}, \zstar) 
\right]
\label{eq:expected-inf}
\end{align}

\begin{theorem}[Influence Function Approximation]
Under the following assumptions: 
\textbf{(1) Model Approximation:} $\theta_k^{(\seqtstart)} \approx \theta_T$ for all $\start$ and $k = 0, \ldots, T-1$; \textbf{(2) Hessian Approximation:} $\hessian_k^{(\seqtstart)} \approx \expectedHessian := \frac{1}{T} \sum_{z \in \dataset} \nabla^2 \ell(\theta_T, z)$ for all $k$; \textbf{(3) Constant Learning Rate:} $\eta_t = \eta$ for all $t = 0, \ldots, T-1$; the expected influence score in Equation \eqref{eq:expected-inf} simplifies and converges to the standard influence function formulation for large $T$:
\begin{align*}
(\ref{eq:expected-inf}) \approx 
\nabla \ell(\theta_T, \zval)^\top 
\left( \sum_{z \in \dataset} \nabla^2 \ell(\theta_T, z) \right)^{-1} \nabla \ell(\theta_T, \zstar)
\end{align*}
\label{thm:inf-approximates-LOO}
\end{theorem}
\textbf{Implications.} The derivation demonstrates that, under the stated approximations, influence function effectively approximates the expected influence score derived from the unrolling differentiation approach as $T$ becomes large. This approximation indicates that the influence function may not fully represent the true leave-one-out score because it relies on simplifying assumptions—such as approximating all model checkpoints and Hessian matrices to be identical—that often do not hold in practical training scenarios. In real-world settings, model parameters evolve significantly throughout training, learning rates are typically scheduled to change over time, and the Hessian matrices can vary considerably between iterations. These factors undermine the validity of the assumptions, thereby limiting the effectiveness of the influence function as an approximation for the leave-one-out score. 
% Moreover, empirical studies have found that influence function scores are often highly noisy in practice \citep{basu2020influence}. Consequently, producing reliable results typically necessitates multiple model retrainings \citep{deng2024efficient}, which can undermine the original computational efficiency advantage of influence functions.
\begin{proof}
Assume that all we have access to is the final model checkpoint $\theta_T := \theta_T^{(\seqtrealization)}$ for a specific realization where $\start = t_r \sim [T-1]$. Under this assumption, the best approximation we can make is: \begin{align*} \theta_T^{(\seqtstart)} \approx \theta_k^{(\seqtstart)} \approx \theta_T \end{align*} for any $\start$ and $k = 0, \ldots, T-1$. Additionally, we approximate the Hessian matrices as: \begin{align} \hessian_k^{(\seqtstart)} \approx \expectedHessian := \frac{1}{T} \sum_{z \in \dataset} \nabla^2 \ell(\theta_T, z) \end{align} and assume a constant learning rate $\eta_t = \eta$ for all $t = 0, \ldots, T-1$.

With these approximations, Equation \eqref{eq:expected-inf} simplifies to:
\begin{align}
(\ref{eq:expected-inf})
&= 
\E_{\start \sim [T-1]}
\left[ 
\eta \nabla \ell(\theta_T, \zval)^\top \left[ \prod_{k=\start+1}^{T-1} (\iden - \eta \expectedHessian) \right] \nabla \ell(\theta_{T}, \zstar) 
\right] \\
&=
\eta \nabla \ell(\theta_T, \zval)^\top 
\E_{\start \sim [T-1]}
\left[ (\iden - \eta \expectedHessian)^{T-1-\start} \right] \nabla \ell(\theta_{T}, \zstar) \\
&= 
\frac{\eta}{T} \nabla \ell(\theta_T, \zval)^\top 
\sum_{\start=0}^{T-1} 
\left( (\iden - \eta \expectedHessian)^{T-1-\start} \right) \nabla \ell(\theta_{T}, \zstar) \\
&\approx 
\frac{\eta}{T} \nabla \ell(\theta_T, \zval)^\top 
\left( \eta \expectedHessian \right)^{-1} \nabla \ell(\theta_{T}, \zstar) \\
&= \nabla \ell(\theta_T, \zval)^\top 
\left( \sum_{z \in \dataset} \g^2 \ell(\theta_T, z) \right)^{-1} \nabla \ell(\theta_{T}, \zstar)
\end{align}
The approximation in the fourth step arises from summing the geometric series of matrices. For $\eta$ sufficiently small and $T$ large, we have $\sum_{s=0}^{T-1} (\iden - \eta \expectedHessian)^{s} \approx (\eta \expectedHessian)^{-1}$. 
\end{proof}

\clearpage

\section{Additional Experiments}
\label{appendix:eval}

\subsection{Baseline \& Implementation Details}
\label{appendix:baselines}

\textbf{Fidality Evaluation (Section \ref{sec:eval-fidelity}).} 
Given the computational intensity, we conduct our experiments on a subset (10\%) of the MNIST using an MLP with two layers with 128 neurons in the hidden layer. We train the model with standard SGD with a learning rate $10^{-2}$ for 10 epochs. We randomly pick 100 data points and compute their ground-truth trajectory-specific LOO score. For the single epoch removal, we remove the data point from the last epoch. 
For this experiment, we do not use random projection and use the full gradients. For the comparison with influence function, we use the state-of-the-art implementation from LoGRA \citep{choe2024your} with the damping term set to be $10^{-3}$ following \citet{bae2022if}. 

\textbf{Large-scale Experiments in Section \ref{sec:eval-efficiency}, \ref{sec:eval-analysis} and \ref{sec:eval-qualitative}.}
Our experiments focus on two language models: Pythia-410M and GPT2-Small. We train these models on two commonly used datasets in the literature for large-scale language model training: \textbf{(1)} A 1\% subset of the Pile dataset \citep{gao2020pile}, and \textbf{(2)} Wikitext-103 \citep{merity2016pointer}. We note that our choice of model architecture size is primarily constrained by the available GPUs in our current setup. However, this limitation does not diminish the significance of our findings. With enough computational resources (e.g., 8 H100 GPUs), our method is readily applicable to perform data attribution for billion-scale model training. 

For both settings, the sequence length is set to 1024. The learning rate is set at a maximum of $3\times 10^{-4}$. We use AdamW as the optimizer with a weight decay of 0.1, and beta values set to 0.9 and 0.95. Gradients are clipped at a maximum value of 1.0 to maintain stability during training. The batch size is set to 16, with a learning rate warmup of 2000 iterations followed by cosine decay. 

For all experiments, for storage reasons, we compute and store projected gradients and data value embedding on linear layers of the model only, with the projection dimension set to 1024 per layer. However, we stress that this is not restricted by computation but by disk storage limit.

\subsection{Additional results for Fidelity Evaluation (Section \ref{sec:eval-fidelity})}
\label{appendix:eval-fidality}

\textbf{Evaluation on more epochs.} Here, we show additional results for the fidelity experiment in Section \ref{sec:eval-fidelity}, where the model is being trained for a longer time (10 epochs), in which case the model is closer to convergence. Figure \ref{fig:groundtruth-LOO-comparison-10-epochs} shows that even in this case, data value embedding still has a high Spearman correlation with the ground-truth LOO in both settings, and the influence function exhibits almost no correlation with the LOO score. 

\begin{figure}[h]
    \centering
    \setlength\intextsep{0pt}
    \setlength\abovecaptionskip{0pt}
    \setlength\belowcaptionskip{0pt}
    \includegraphics[width=\textwidth]{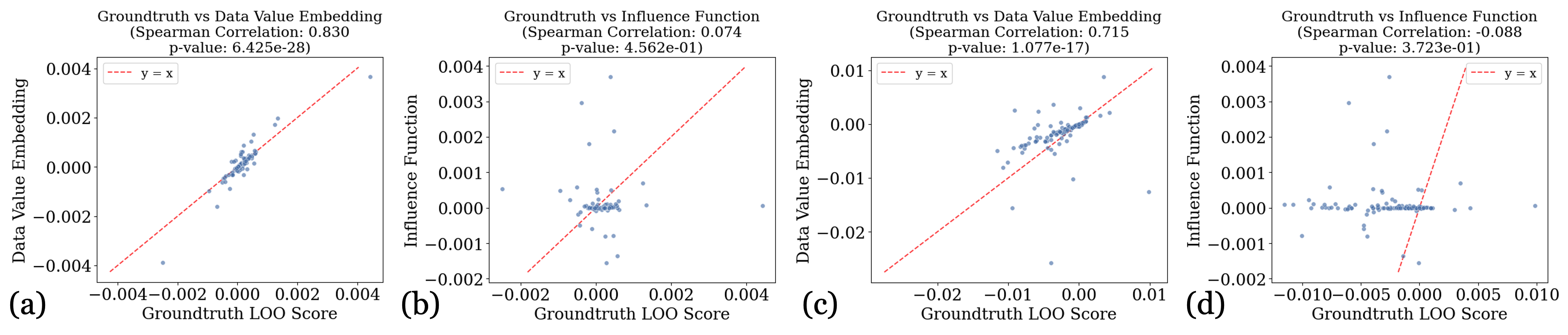}
    \caption{
    The correlation between ground-truth LOO when the MLP is trained for \textbf{10 epochs} and the estimation obtained by (a) the data value embedding method and (b) the influence function for \emph{single epoch removal}. (c) and (d) present the corresponding correlations for \emph{all-epoch removal}. 
    }
    \label{fig:groundtruth-LOO-comparison-10-epochs}
\end{figure}

\rebuttal{\textbf{Evaluation on different architectures.} To demonstrate that our method's effectiveness extends beyond simple MLPs, we evaluate data value embedding on a CNN architecture consisting of two convolutional layers (with 32 and 64 filters, respectively, each followed by 2x2 max pooling) and a final linear layer. We train the model on MNIST using SGD with learning rate $10^{-2}$ for 10 epochs. Following the same experimental setup as with MLP, we randomly select 100 data points and compute their ground-truth trajectory-specific LOO scores for both single-epoch and all-epochs removal settings. Figure \ref{fig:groundtruth-LOO-comparison-3-epochs-CNN} shows that data value embedding maintains strong correlation with ground-truth LOO scores, achieving Spearman correlations of 0.818 for single-epoch removal (Figure \ref{fig:groundtruth-LOO-comparison-3-epochs-CNN} (a)) and 0.682 for all-epochs removal (Figure \ref{fig:groundtruth-LOO-comparison-3-epochs-CNN} (b)). These results demonstrate that our method can effectively approximate data influence across different neural architectures.}

\begin{figure}[h]
    \centering
    \setlength\intextsep{0pt}
    \setlength\abovecaptionskip{0pt}
    \setlength\belowcaptionskip{0pt}
    \includegraphics[width=0.6\textwidth]{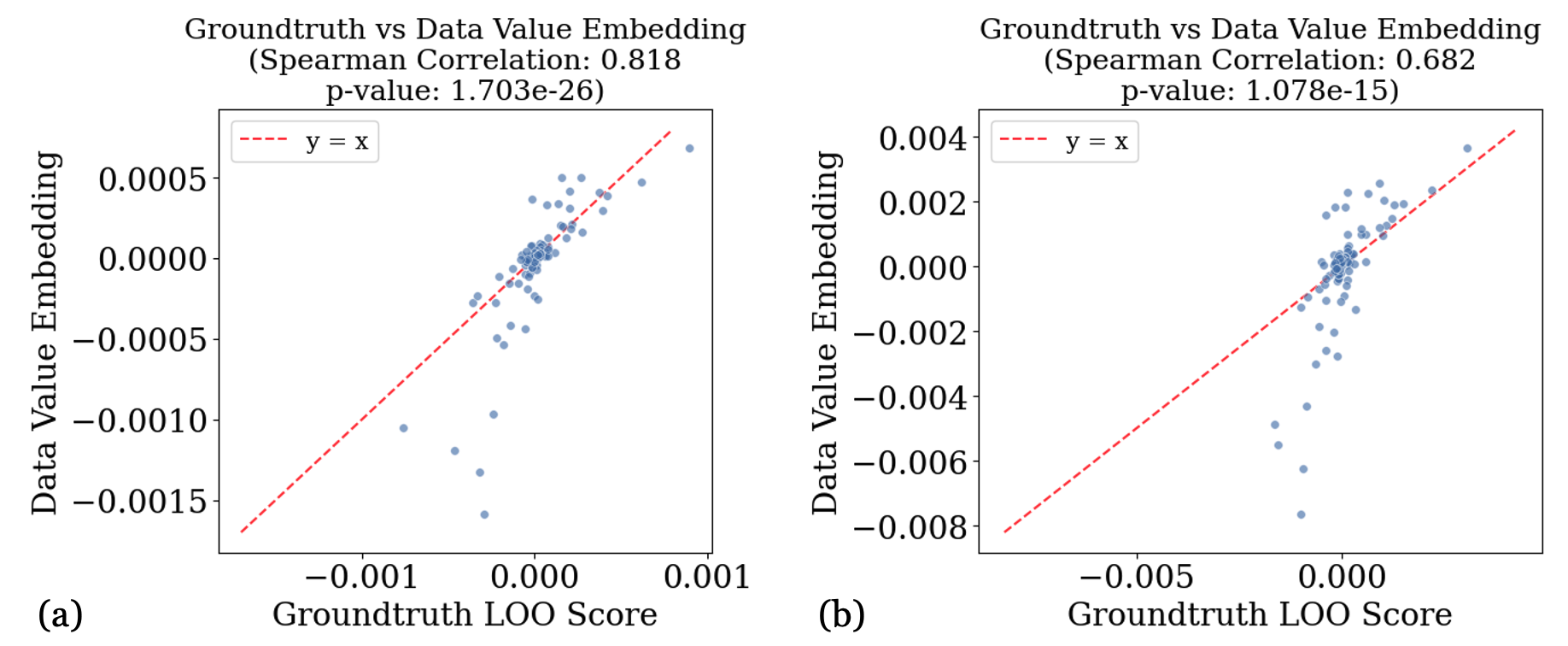}
    \caption{
    Scatter plot showing the correlation between ground-truth LOO and data value embedding method when training a small CNN on MNIST for \textbf{10 epochs}, where (a) is for \emph{single epoch removal} (b) for \emph{all-epoch removal} setting. 
    }
    \label{fig:groundtruth-LOO-comparison-3-epochs-CNN}
\end{figure}

\subsubsection{\rebuttal{Effectiveness on Mislabeled Data Detection and Data Selection}}
\label{appendix:mislabel-detection}

In addition to comparing our results with ground-truth training run-specific LOO, we further evaluate the performance of our data value embedding algorithm in the task of mislabeled data detection \rebuttal{and data selection}, two standard benchmarks in data attribution literature. We compare several data attribution baselines, including \emph{Retraining-based Data Shapley \citep{ghorbani2019data}, KNN-Shapley \citep{jia2019efficient}, Influence Function \citep{koh2017understanding}, Trak \citep{park2023trak}, Empirical Influence Functions \citep{feldman2020neural}, and Datamodels \citep{ilyas2022datamodels}}. 

\textbf{Experiment settings.} We use ImageNet-pretrained ResNet18 as the architecture in the experiment. Given the computational intensity of retraining-based methods, we conduct our experiments on a subset of 1,000 samples from CIFAR-10 dataset. 
We use Adam with a learning rate 0.001, weight decay of 1e-4, and label smoothing of 0.1 over 50 epochs. The learning rate is reduced by a factor of 0.1 every 10 epochs. The batch size is set to 64. For retraining-based techniques (Retraining-based Data Shapley, Empirical Influence Functions, Datamodels), we estimate the corresponding attribution scores with 1000 model training runs. For Trak, we set the projection dimension to be 2048. For KNN-Shapley, we set $K=5$ and use the features extracted from the last linear layer of ResNet18. Both experiments included 10\% random label noise to reflect the challenges of real-world data.

\begin{remark}
The mislabeled data detection \rebuttal{and data selection} benchmark here mainly serves to evaluate the fidelity of our algorithm in settings where ground-truth LOO computation is infeasible. However, we stress that data value embedding is not specifically designed for those tasks. Rather, it is developed as an interpretability tool and a mechanism for real-time data valuation, with potential applications in data marketplaces and addressing AI copyright disputes \citep{wang2024economic}. 
\end{remark}

\textbf{I. Mislabeled Data Detection.}

Table \ref{tb:mislabel-detection} shows that KNN-Shapley achieves the highest accuracy in detecting mislabeled data, likely due to its sensitivity to label inconsistencies. \emph{Retraining-based methods} (Retraining-based Data Shapley, Empirical Influence Functions, Datamodels) exhibit the lowest performance, which can be attributed to the inefficiency of Monte Carlo sampling and the inherent stochasticity during retraining, as discussed in \citet{wang2023data}. Among techniques requiring only a single training run, Trak underperforms relative to other methods. This observation aligns with findings from its original paper \citep{park2023trak}, which suggests that ensemble methods are often necessary for optimal performance. Notably, data value embedding and influence function achieve comparable performance, outperforming all other techniques except KNN-Shapley. The strong performance of these methods likely stems from their deterministic nature, which provides more consistent and reliable results.

\begin{table}[t]
\centering
\resizebox{0.75\columnwidth}{!}{\begin{tabular}{@{}cc@{}}
\toprule
\textbf{Method} & \textbf{Performance (Mean ± Std)} \\ 
\midrule
\textbf{Data Shapley} \citep{ghorbani2019data}  & 0.582 (0.029) \\
\textbf{Empirical Influence Function} \citep{feldman2020neural}  & 0.552 (0.017) \\
\textbf{Datamodels} \citep{ilyas2022datamodels}                  & 0.520 (0.008) \\ 
\textbf{KNN-Shapley} \citep{jia2019efficient, wang2023noteknn}   & 0.760 (0.018) \\
\textbf{Trak} \citep{park2023trak}                               & 0.511 (0.012) \\
\textbf{Influence Function} \citep{koh2017understanding}          & 0.654 (0.054) \\ 
\textbf{Data Value Embedding (ours)}                             & 0.667 (0.031) \\ 
\bottomrule
\end{tabular}
}
\vspace{-2mm}
\caption{
AUROC scores of mislabeled data detection task with various data attribution techniques on CIFAR10 dataset. The higher the AUROC score is, the better the method is. The results are across three different training runs (the randomness comes from construction of corrupted datasets), where we show the standard deviation in ().
}
\label{tb:mislabel-detection}
\end{table}

\rebuttal{\textbf{II. Data Selection.}}

\rebuttal{
Table \ref{tb:data-selection} demonstrates that Data Value Embedding outperforms \emph{all} existing data valuation methods in the task of data selection. \emph{Retraining-based methods} (Data Shapley, Empirical Influence Functions, Datamodels) show limited effectiveness due to the high variance introduced by Monte Carlo sampling and learning stochasticity. While the influence function and Trak do not require model retraining, their performance is constrained by assumptions that often do not hold in practice, such as model convergence and strong convexity. KNN-Shapley provides stable valuation results. However, it assigns similar scores to similar data points, potentially reducing dataset diversity among the selected data subset. In contrast, Data Value Embedding considers both data characteristics and temporal ordering in training, allowing similar data points to receive different scores based on when they appear in the training sequence. This temporal awareness helps maintain dataset diversity while identifying valuable samples.
}

% Table \ref{tb:data-selection} shows that Data Value Embedding works significantly better than all existing data valuation methods, as those retraining-based data attribution methods are highly noisy. The influence function and Trak are not very reliable given the unrealistic assumption. While KNN-Shapley provides deterministic valuation results, it may easily loose the diversity as similar data points are guaranteed to have similar value scores \citep{wang2024rethinking}, but since Data Value Embedding additionally considers the data ordering, similar data points are not necessarily receive the same score, thus help with diversity. 

\begin{table}[h]
\centering
\resizebox{\columnwidth}{!}{\begin{tabular}{@{}ccccc@{}}
\toprule
\textbf{}                              & \textbf{20\%}           & \textbf{40\%}           & \textbf{60\%}           & \textbf{80\%}           \\ \midrule
\textbf{Random}                        & 0.350 (0.010)          & 0.461 (0.010)          & 0.525 (0.004)          & 0.559 (0.003)          \\
\textbf{Data Shapley} \citep{ghorbani2019data}                  & 0.317 (0.047)          & 0.468 (0.010)          & 0.527 (0.004)          & 0.570 (0.008)          \\
\textbf{Empirical Influence Function} \citep{feldman2020neural}  & 0.342 (0.004)          & 0.466 (0.016)          & 0.530 (0.009)          & 0.568 (0.010)          \\
\textbf{Datamodels} \citep{ilyas2022datamodels}                    & 0.342 (0.004)          & 0.465 (0.004)          & 0.534 (0.010)          & 0.559 (0.005)          \\
\textbf{KNN-Shapley} \citep{jia2019efficient}                   & 0.354 (0.017)          & 0.478 (0.007)          & 0.525 (0.015)          & 0.563 (0.005)          \\
\textbf{Trak} \citep{park2023trak}                          & 0.329 (0.021)          & 0.443 (0.030)          & 0.517 (0.016)          & 0.572 (0.009)          \\
\textbf{Influence function} \citep{koh2017understanding}                           & 0.320 (0.033)          & 0.450 (0.028)          & 0.530 (0.015)          & 0.580 (0.004)          \\
\textbf{Data Value Embedding (ours)}          & \textbf{0.391 (0.007)} & \textbf{0.518 (0.008)} & \textbf{0.566 (0.005)} & \textbf{0.604 (0.009)} \\
\bottomrule
\end{tabular}}
\vspace{-2mm}
\caption{
Test accuracies when training ResNet18 on high-value data points selected by various data attribution techniques. To be able to compare with techniques that require model retraining, for each training run we randomly sample a size-1000 subset of CIFAR10 dataset (with 10\% data points being mislabeled). The results are across three different training runs (the randomness comes from construction of corrupted datasets), where we show the standard deviation in ().
}
\label{tb:data-selection}
\end{table}

\subsection{Additional discussion and results for Section \ref{sec:eval-analysis}}
\label{appendix:eval-dynamics}

\subsubsection{Explanation of influence trend}

\textbf{1. Parameter initialization and warmup training are important for the final model performance.} The blue curve in Figure \ref{fig:forgetting-and-intuition} (b) illustrates the trend of average training data gradient norm throughout the training process. We observe that gradient norms are typically large and unstable during early training ($\start \le 2000$). As training progresses, these norms decrease rapidly, leading to a significant reduction in the eigenvalues of the Hessian matrix $\hessian_t \approx \sum_{z \in \batcht} \g \ell(\theta_{t}, z)\g \ell(\theta_{t}, z)^\top$. Consequently, when $\norm{\g \ell(\theta_{\start})}$ is significantly larger than later training gradients, the norm of data value embedding $\norm{\prod_{t=\start+1}^T (\iden - \eta_t \hessian_t) \g \ell(\theta_{\start})}$ remains substantial. This results in early-stage data points maintaining significant influence until the end of training. Figure \ref{fig:forgetting-and-intuition} (a) further illustrates this phenomenon. The \rebuttal{purple} curve shows that training data points from the High-impact Warmup Phase, while experiencing large drops in influence as training progresses, still maintain higher influence than later data points. This observation aligns with the well-established effect that model initialization and/or warm-up training plays a crucial role in training performance, effectively initializing model parameters and gradually preparing the model for more complex learning tasks. 

\textbf{2. Influence saturation from future data.} As the model enters a relatively smooth loss regime ($\start > ~2000$ in Figure \ref{fig:forgetting-and-intuition} (b)), the training data gradient norm decreases much more slowly. In this phase, the magnitude deflation effect from $\prod_{t=\start}^T (\iden - \eta_t \hessian_t)$ remains significant for relatively small $\start$, while the training gradient norm $\norm{\g \ell(\theta_{\start})}$ does not differ significantly between earlier and later training points. This results in $\norm{\prod_{t=\start}^T (\iden - \eta_t \hessian_t) \g \ell(\theta_{\start})} < \norm{\g \ell(\theta_{\ta})}$ for $\ta > \start$, creating a low-impact basin during the early-to-middle training stage. In this basin, influence scores are lower than those of data points from both the very early and later training stages. The \rebuttal{red} curve in Figure \ref{fig:forgetting-and-intuition} (a) demonstrates this trend, showing influence scores for these points gradually decreasing during training and eventually falling below those of later training data points. This pattern aligns with the phenomenon of catastrophic forgetting, where the model appears to "forget" the influence of data from this middle phase as training progresses. 
One might initially think this phenomenon is connected to catastrophic forgetting, where the model appears to "forget" the influence of data from earlier training phases as it progresses. However, we note that a data point's influence score decreases the most when future data points are similar to it, which is different from catastrophic forgetting. Intuitively, if future points are identical, the presence of the earlier data point in training becomes less relevant to the model's behavior.

In Figure 7(b), we consider a simplified setting where we approximate Hessian with GGN matrix and assume all training gradients are orthogonal across different iterations. Under these assumptions, $\hessian_k \approx \bm{G}_k = \sum_{z\in B_k} \nabla\ell(\theta_k,z)\nabla\ell(\theta_k,z)^\top$ becomes a sum of rank-1 matrices that have non-overlapping eigenspaces. Given the orthogonality assumption, we have $\bm{G}_t \bm{G}_s = 0$ for $t \neq s$, and the product $\prod_{t=t_s}^T (I - \eta_t \bm{G}_t)$ simplifies to $I - \sum_{t=t_s}^T \eta_t \bm{G}_t$. Since each $\bm{G}_t = \sum_{z \in B_t} \nabla \ell(\theta_t, z) \nabla \ell(\theta_t, z)^\top$ is a sum of rank-1 matrices along orthogonal directions, the trace of this product can be analytically computed as $p - \sum_{t=t_s}^T \eta_t \sum_{z \in B_t} \norm{\nabla \ell(\theta_t, z)}^2$, where $p$ is the dimension of parameter space. Furthermore, if we assume $\nabla \ell(\theta_{t_s})$ follows a Gaussian distribution, then $\norm{(I - \sum_{t=t_s}^T \eta_t G_t)\nabla \ell(\theta_{t_s})}$ follows a scaled chi distribution since it's the norm of a Gaussian vector after linear transformation by an orthogonal projection matrix. This enables us to analytically compute its expected value, as shown by the green curve in Figure \ref{fig:forgetting-and-intuition} (b).

\begin{figure}[h]
    \centering
    \setlength\intextsep{0pt}
    \setlength\abovecaptionskip{0pt}
    \setlength\belowcaptionskip{0pt}
    \includegraphics[width=\textwidth]{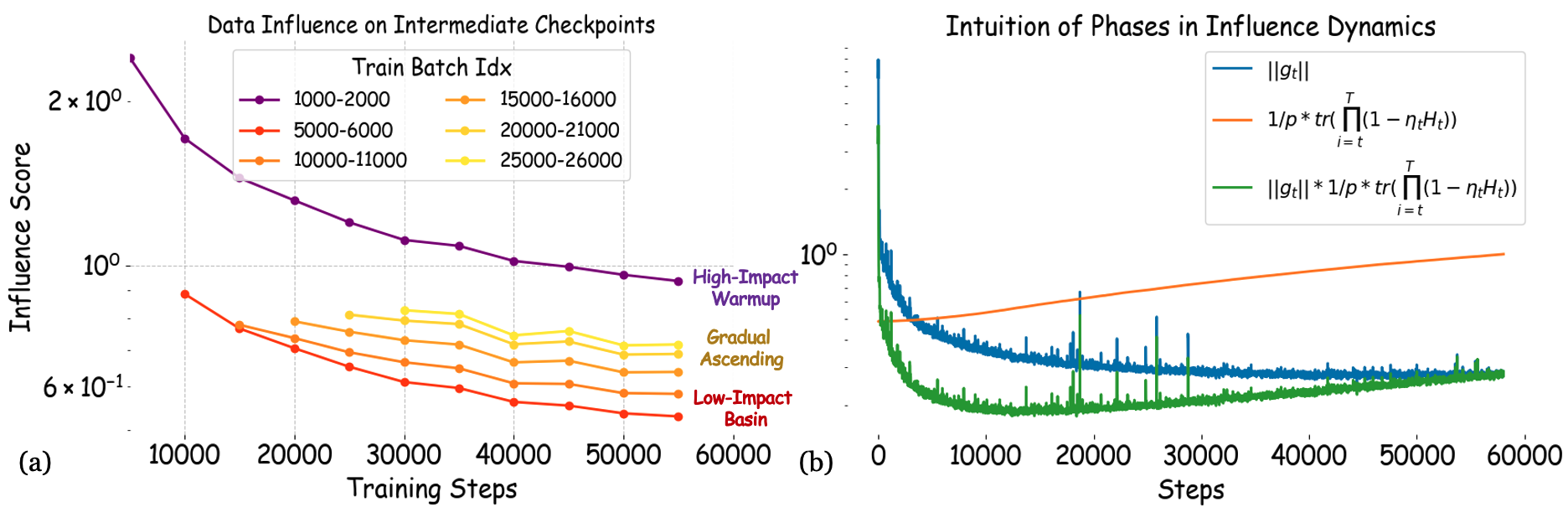}
    \caption{
    \textbf{(a)} (same as Figure \ref{fig:influence-saturation} in the main paper) Influence scores of data points from different training stages on intermediate model checkpoints throughout training. The x-axis denotes the number of training iterations, and the y-axis represents the influence score of selected data points on the model at each checkpoint.
    \textbf{(b)} The blue curve shows the average gradient changes as model training progresses. 
    The orange and green curves are analytical curves under a simplified setting, where the orange curve is the analytical trace of $\prod_{t=\start}^T (\iden - \eta_t \hessian_t)$ as $\start$ increases, and the green curve shows the norm of data value embedding for Gaussian-distributed gradient under this simplified setting. 
    }
    \label{fig:forgetting-and-intuition}
\end{figure}

\subsubsection{Additional Details for Figure \ref{fig:influence-plot-pythia-Pile} (b)}
\label{appendix:eval-dataselection-detail}

\rebuttal{In Figure 1(b), we compare different strategies for applying online data selection during model training. The online selection process identifies high-quality training batches by (1) sampling a candidate pool of training points with size 2B, where B is the desired batch size, (2) computing the gradient cosine similarity between each candidate point and a small validation batch (randomly sampled from the full validation set), and (3) selecting the B points with the highest similarity scores to form the next training batch. This procedure incurs significant computational overhead, requiring additional forward and backward passes for similarity computation at each selection step. When not performing online selection (i.e., during the "random selection" phases), we simply sample training batches randomly. Notably, the model processes the entire training dataset regardless of the selection strategy - what varies is only how batches are prioritized during different training phases. The "Early+Late" strategy applies online selection only during iterations 1-2000 and after iteration 20000, while using random selection in between. This selective approach achieves 96\% of the performance improvement of continuous selection while reducing the computational overhead by more than 5×, suggesting that precise batch selection is most critical during the early and late training phases.}

\subsubsection{Additional Results}

Figure \ref{fig:dynamics-gpt2-pretrain}
presents additional results on the data influence scores of training data across different stages of LLM pretraining, using more datasets and model architectures. We observe that the data influence scores on the final model can consistently be categorized into three distinct regimes throughout pretraining.

Figure \ref{fig:dynamics-gpt2-finetune} shows the results when using pretrained models downloaded from Huggingface. In this scenario, the situation diverges across different datasets. Notably, when we continually pretrain on the Pile dataset, no gradual ascending phase is observed at the end. However, when GPT-2 has already been pre-trained, continuing pretraining on Wikitext-103 once again exhibits a gradual ascending phase. This is likely because Wikitext-103 is a relatively small dataset, and fine-tuning it for three epochs can easily lead to overfitting, as illustrated in Figure \ref{fig:loss-curve} (d).

\begin{figure}[h]
    \centering
    \setlength\intextsep{0pt}
    \setlength\abovecaptionskip{0pt}
    \setlength\belowcaptionskip{0pt}
    \includegraphics[width=0.75\textwidth]{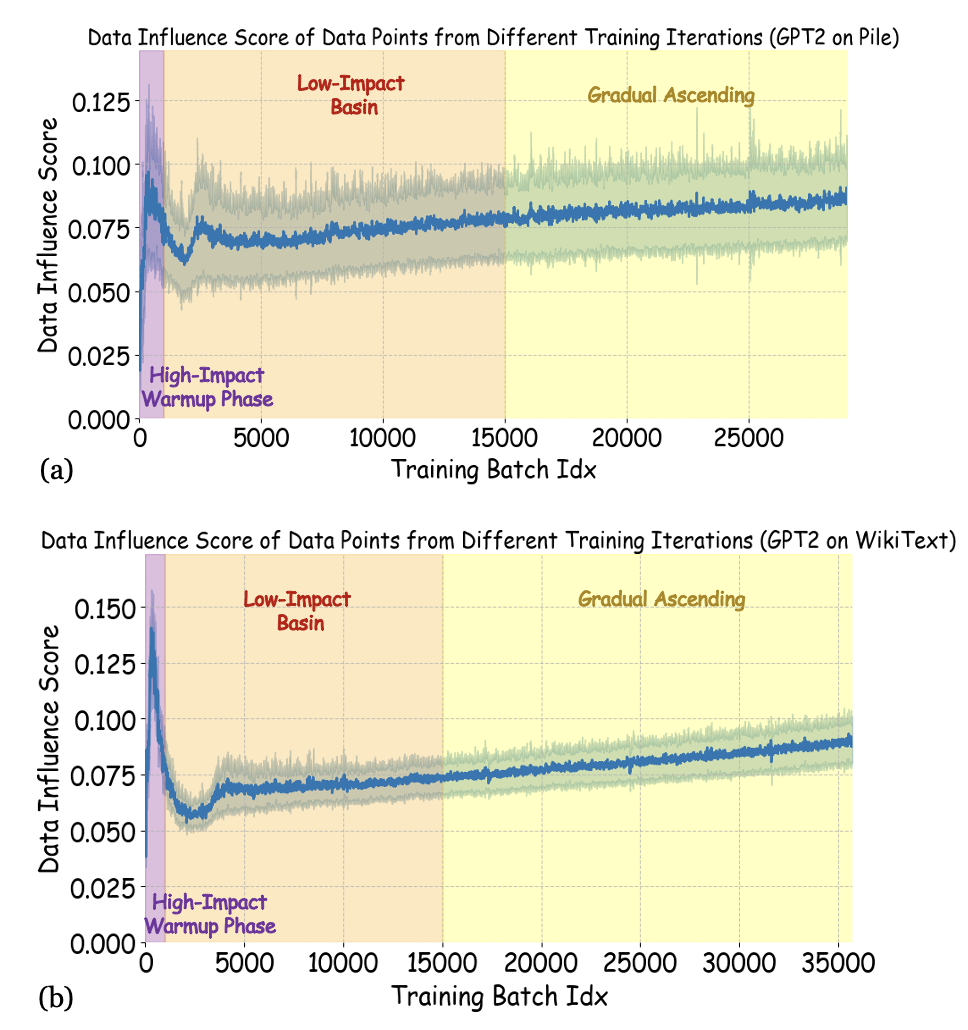}
    \caption{
    Average data influence scores per training batch, measured against the final model's loss 
    where (a) GPT2 trained on 1\% of Pile, and (b) GPT2 trained on WikiText-103M for 3 epochs.
    }
    \label{fig:dynamics-gpt2-pretrain}
\end{figure}

\begin{figure}[h]
    \centering
    \setlength\intextsep{0pt}
    \setlength\abovecaptionskip{0pt}
    \setlength\belowcaptionskip{0pt}
    \includegraphics[width=0.75\textwidth]{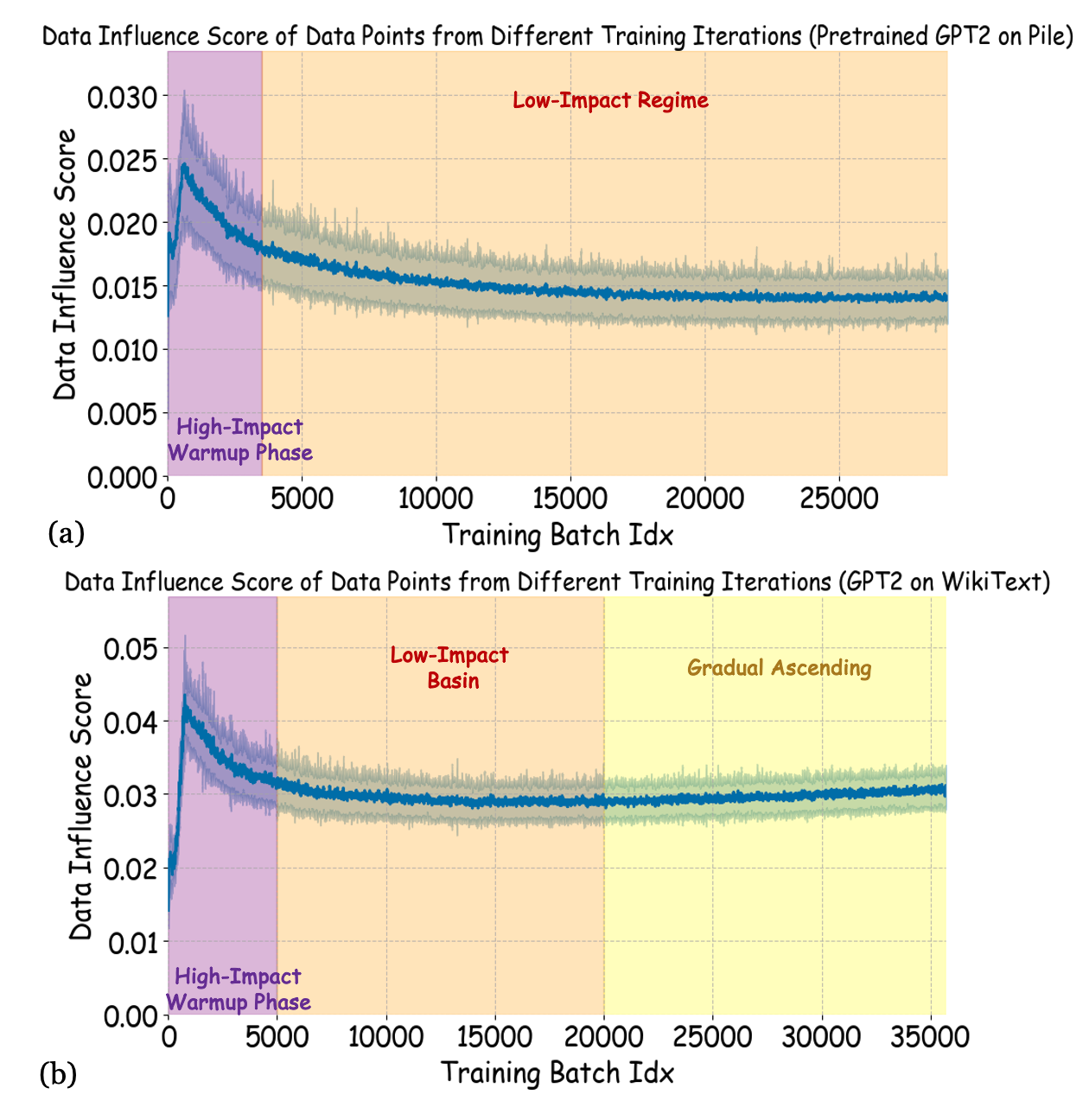}
    \caption{
    Average data influence scores per training batch, measured against the final model's loss 
    where (a) Pretrained GPT2 trained on 1\% of Pile, and (b) Pretrained GPT2 trained on WikiText-103M for 3 epochs.
    }
    \label{fig:dynamics-gpt2-finetune}
\end{figure}

\begin{figure}[h]
    \centering
    \setlength\intextsep{0pt}
    \setlength\abovecaptionskip{0pt}
    \setlength\belowcaptionskip{0pt}
    \includegraphics[width=0.75\textwidth]{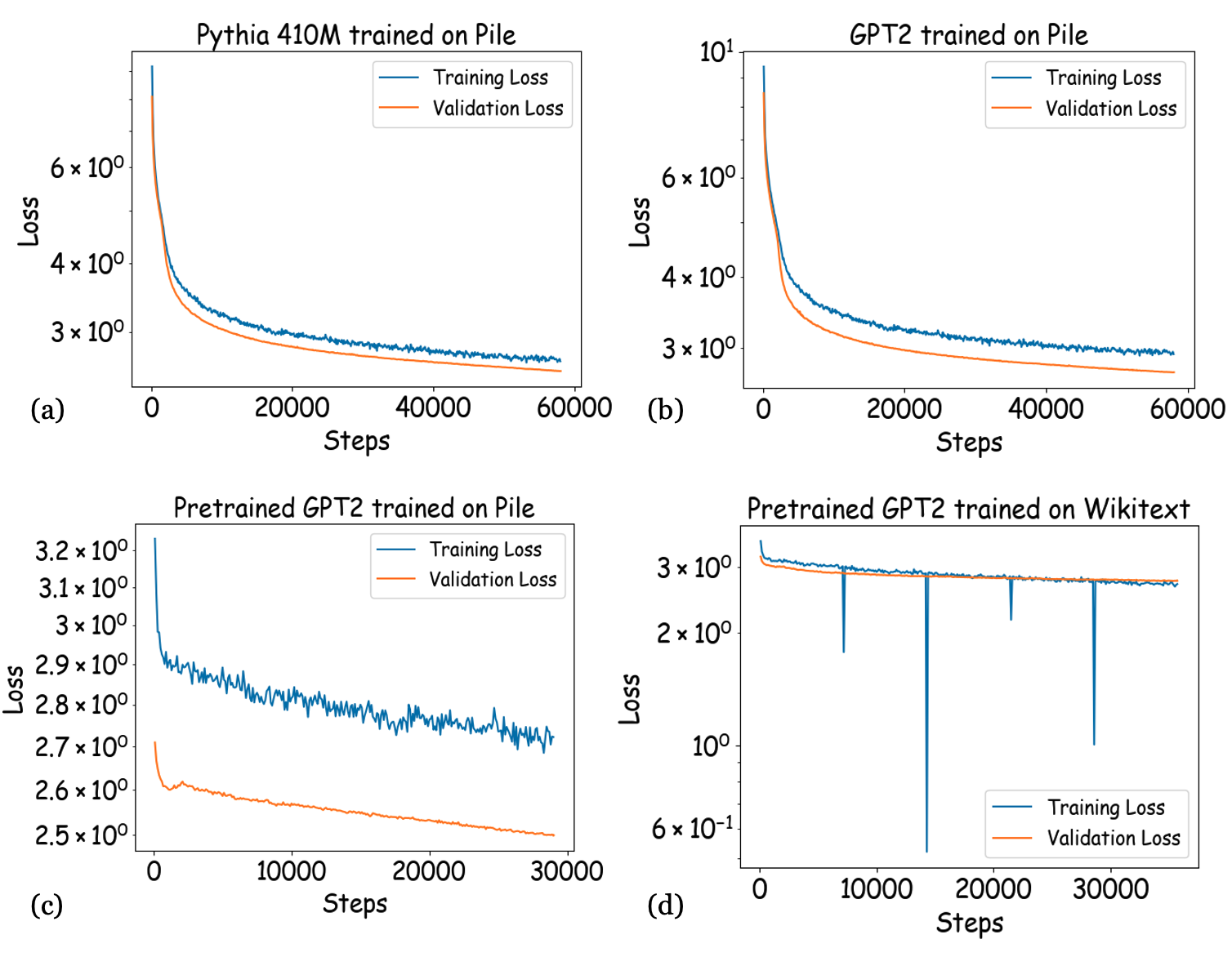}
    \caption{
    Loss curve for the training.
    }
    \label{fig:loss-curve}
\end{figure}

\clearpage

\subsection{Ablation Study: Error from Projection Dimension}
\label{appendix:eval-projection-error}

We examine the error introduced by the random projection of gradient vectors, as discussed in Section \ref{sec:store-gradient}. Specifically, we evaluate the Spearman correlation between data influence scores when using per-layer projection dimensions in $\{256, 1024, 2304\}$ and compare these to a larger per-layer projection dimension of $4096$. Since computing the ground truth without projection is infeasible, we use projection dimension $4096$ as a reference point for comparison. Additionally, we compare our results to LoGRA \citep{choe2024your}, the most efficient current implementation of the influence function, which also employs random projection to store data attribution information. Due to the computational and disk storage constraints, these experiments were conducted using GPT-2, trained on 5\% of the Wikitext-103 dataset. The results, shown in Figure \ref{fig:projection-error} (a), indicate that our data value embedding method achieves a higher Spearman correlation compared to the influence function. While our results demonstrate a clear advantage, a more in-depth analysis of the observed improvements would be an interesting direction for future research.

\begin{figure}[h]
    \centering
    \setlength\intextsep{0pt}
    \setlength\abovecaptionskip{0pt}
    \setlength\belowcaptionskip{0pt}
    \includegraphics[width=\textwidth]{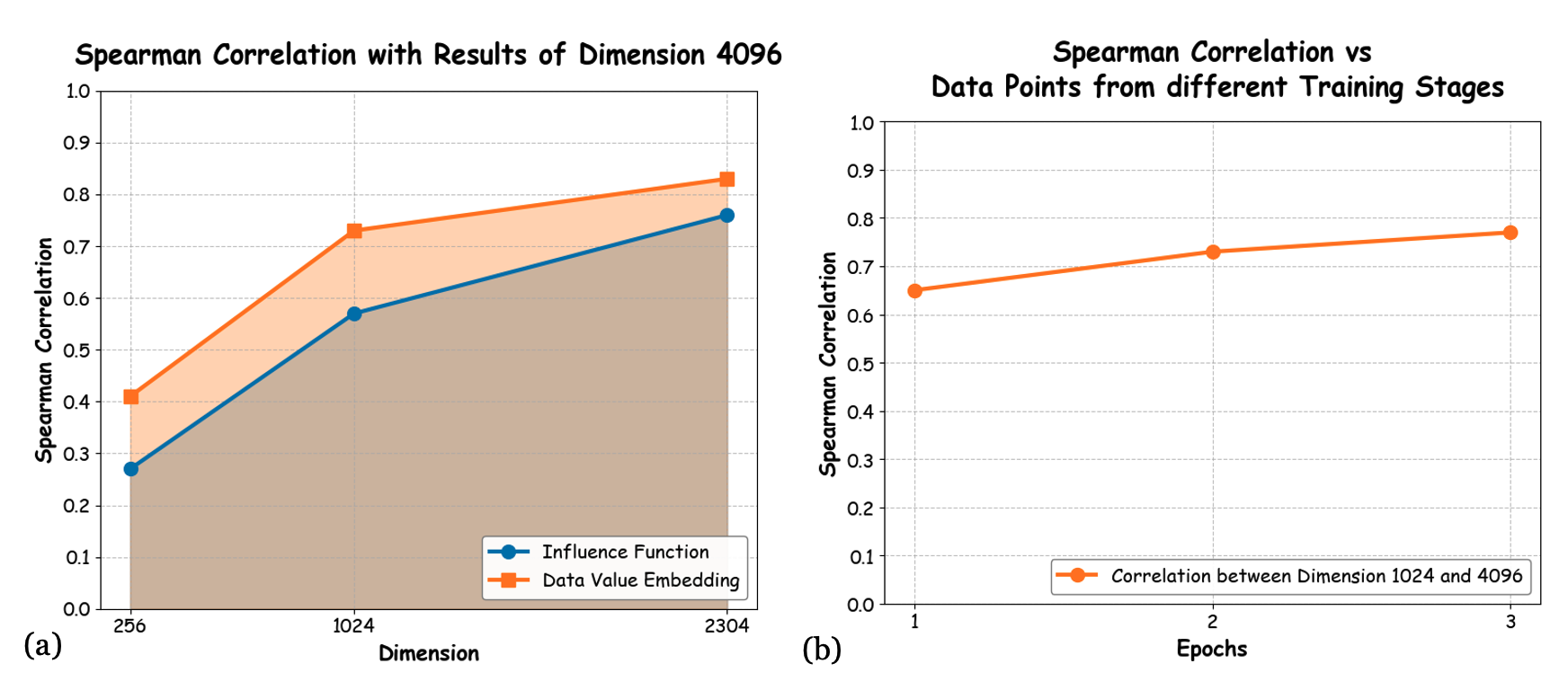}
    \caption{
    (a) Comparison of Spearman correlation between data influence scores as a function of projection dimension. 
    }
    \label{fig:projection-error}
\end{figure}

\subsection{\rebuttal{Examples of Bad Data}}
\label{appendix:eval-baddata}

\rebuttal{To demonstrate Data Value Embedding's capability in identifying potentially problematic training examples, we examined the training data points from the Pile dataset \citep{gao2020pile} that received the most negative influence scores under the same experiment settings in Section \ref{sec:eval-analysis}. Figure \ref{fig:bad-examples-1} and \ref{fig:bad-examples-2} show these examples and their influence scores.}

\rebuttal{Our analysis revealed several types of training data that could potentially harm model performance. First, we find quite a few code-related samples that, while syntactically valid, provide minimal educational value for language modeling. These include YAML configurations with simple numeric arrays, raw binary data represented as hexadecimal strings, and pixel-by-pixel image data in array format. Such examples contain little information about code structure or programming patterns while introducing noise into the training process.}

\rebuttal{Beyond code snippets, we found examples of text data that could potentially bias the model's learning. For instance, math problems (Figure \ref{fig:bad-examples-1}'s last example) that follow identical question formats could bias the model toward specific phrasings (e.g., \emph{"What is $\ldots$"}) rather than developing diverse language understanding. We also identified articles that, while containing meaningful content about important topics like data privacy, suffer from poor formatting with missing punctuation and paragraph breaks (Figure \ref{fig:bad-examples-2}'s second example). Such poorly formatted content, while topically relevant, could potentially degrade the model's ability to learn proper text formatting, punctuation usage, and document structure.}

\begin{figure}[h]
\centering
\includegraphics[width=\textwidth]{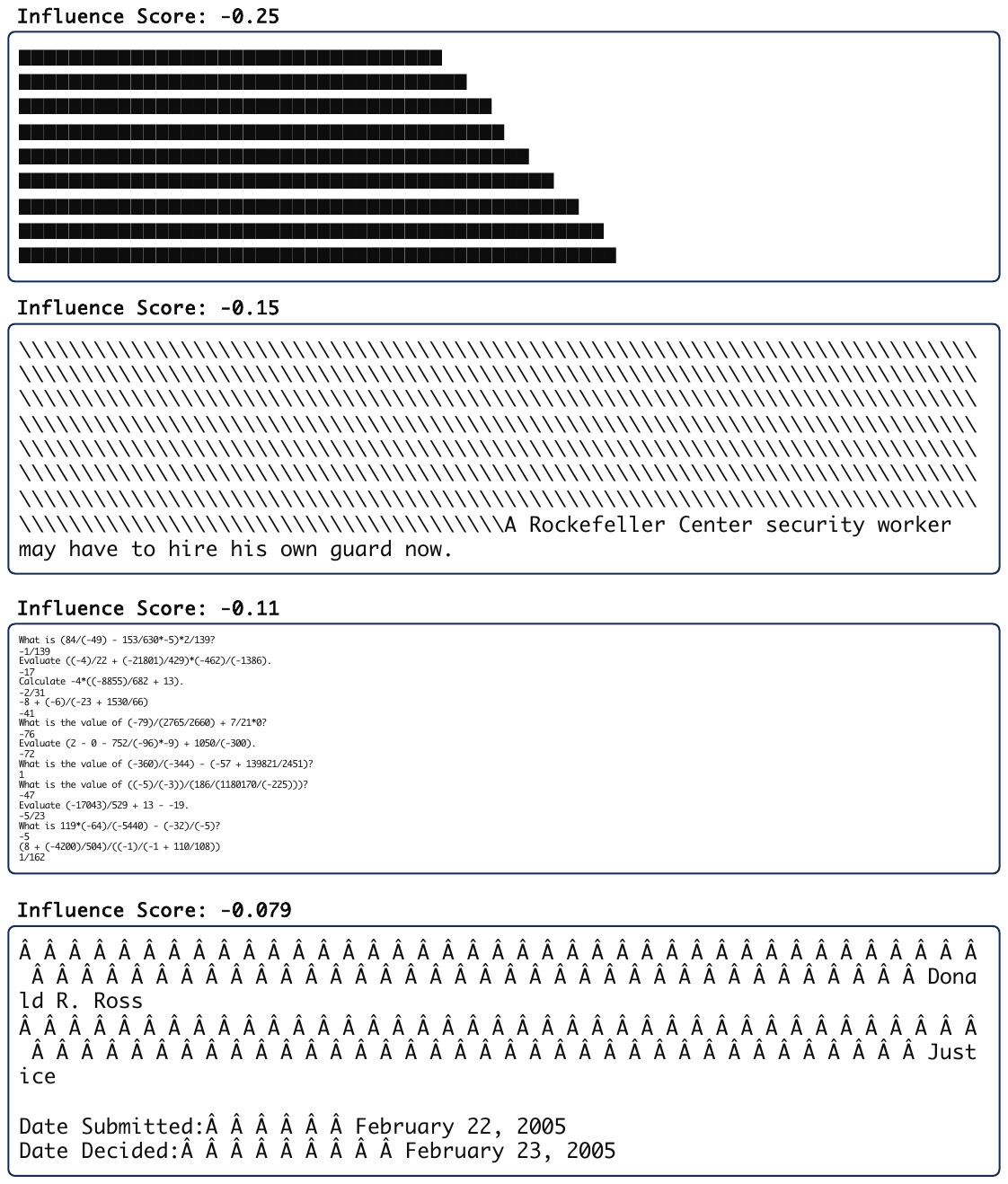}
\caption{Examples of training data from the Pile dataset identified as potentially problematic by our method, along with their influence scores. The examples include configuration files that provide minimal learning value for language modeling, as well as repetitive mathematical problems with identical question formats.}
\label{fig:bad-examples-1}
\end{figure}

\begin{figure}[h]
\centering
\includegraphics[width=\textwidth]{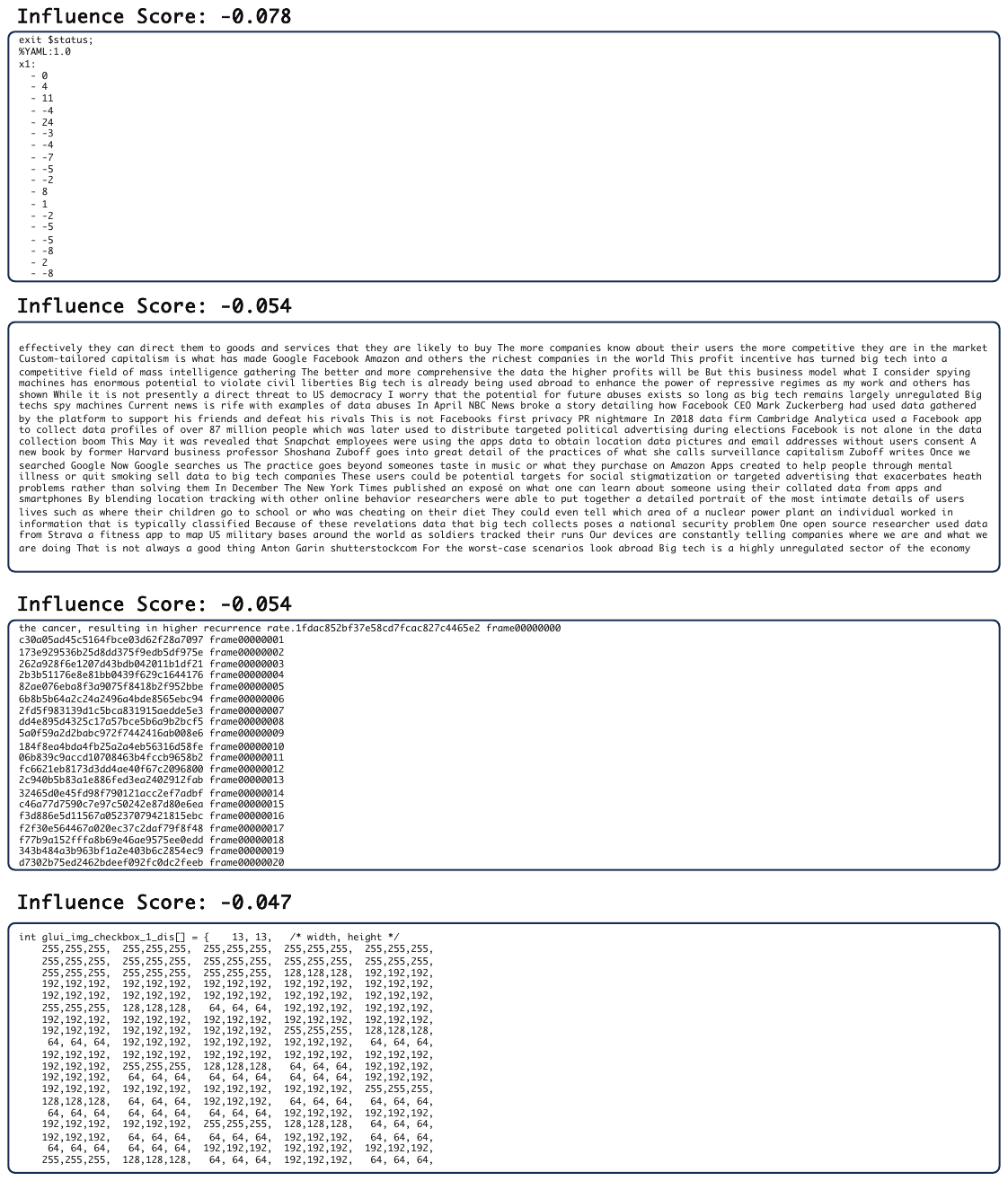}
\caption{Additional examples of training data identified as potentially problematic, showing text content with poor formatting (missing punctuation and paragraph breaks) and low-information code snippets with repetitive numeric arrays.}
\label{fig:bad-examples-2}
\end{figure}

\end{document}